\documentclass{article}
\pdfoutput=1
\usepackage[nonatbib,final]{nips_style}

\usepackage[utf8]{inputenc} 
\usepackage[T1]{fontenc}    
\usepackage{hyperref}       
\usepackage{url}            
\usepackage{booktabs}       
\usepackage{amsfonts}       
\usepackage{nicefrac}       
\usepackage{microtype}      
\usepackage[english]{babel}

\usepackage{wrapfig,tikz}
\usepackage{amsmath,longtable,fancyhdr,booktabs,multirow,graphicx,float}
\usepackage{amssymb,xcolor,amsthm}
\usepackage{subfigure}
\usepackage{color}
\usepackage{colortbl}
\usepackage{theoremref}

\newcommand{\T}{\intercal}
\newcommand{\mbf}[1]{\mathbf{#1}}
\newcommand{\bs}[1]{\boldsymbol{#1}}
\newcommand{\mbb}[1]{\mathbb{#1}}

\newcommand{\ud}{\mathrm{d}}
\newcommand{\up}{\mathrm}

\newcommand{\mcal}{\mathcal}
\newcommand{\norm}[1]{\left\lVert#1\right\rVert}

\newtheorem{theorem}{Theorem}
\newtheorem{corollary}{Corollary}

\newtheorem{proposition}{Proposition}
\newtheorem{assumption}{Assumption}
\newtheorem{innercustomthm}{Theorem}
\newtheorem{definition}{Definition}
\newenvironment{customthm}[1]
  {\renewcommand\theinnercustomthm{#1}\innercustomthm}
  {\endinnercustomthm}

\newtheorem{innercustomprop}{Proposition}
\newenvironment{customprop}[1]
 {\renewcommand\theinnercustomprop{#1}\innercustomprop}
  {\endinnercustomprop}

\newcommand{\be}{\begin{equation}}
\newcommand{\ee}{\end{equation}}

\definecolor{Gray}{gray}{0.85}
\definecolor{LightCyan}{rgb}{0.88,1,1}

\newcolumntype{a}{>{\columncolor{Gray}}c}
\newcolumntype{b}{>{\columncolor{white}}c}

\DeclareMathOperator*{\argmin}{arg\,min}

\DeclareMathOperator*{\arginf}{arg\,inf}

\makeatletter
\usepackage{xspace}
\def\@onedot{\ifx\@let@token.\else.\null\fi\xspace}
\DeclareRobustCommand\onedot{\futurelet\@let@token\@onedot}

\newcommand{\figref}[1]{Fig\onedot~\ref{#1}}

\newcommand{\thmref}[1]{Thm\onedot~\ref{#1}}

\newcommand{\propref}[1]{Prop\onedot~\ref{#1}}

\newcommand{\Ep}{\mcal{E}}
\newcommand{\Epp}{\widehat{\mcal{E}}_s^+[\mu]}
\newcommand{\Eln}{\widehat{\mcal{E}}_{\lambda,n}[\mu]}
\newcommand{\mup}{\widehat{\mu}_{(X,Y)}^+}
\newcommand{\muln}{\widehat{\mu}_{\lambda,n}}
\newcommand{\Hy}{\mcal{H}_{\mcal{Y}}}
\newcommand{\Hx}{\mcal{H}_{\mcal{X}}}
\newcommand{\HK}{\mcal{H}_{K}}
\newcommand{\HH}{\mcal{H}}
\newcommand{\XX}{\mcal{X}}
\newcommand{\YY}{\mcal{Y}}
\newcommand{\ZZ}{\mcal{Z}}
\newcommand{\BB}{\mcal{B}}
\newcommand{\zz}{\mbf{z}}
\newcommand{\xx}{\mbf{x}}
\newcommand{\yy}{\mbf{y}}
\newcommand{\CC}{\mcal{C}}
\newcommand{\hC}{\widehat{\mcal{C}}}

\def\ie{\emph{i.e}\onedot}

\def\wrt{w.r.t\onedot}

\def\aka{a.k.a\onedot}
\def\iid{i.i.d\onedot}

\title{Kernel Bayesian Inference with \\Posterior Regularization}

\author{
	Yang Song$^\dag$, ~ Jun Zhu$^\ddag$\thanks{Corresponding author.}, ~ Yong Ren$^\ddag$ 
	\\
	$^\dag$ Dept. of Physics, Tsinghua University, Beijing, China \\
    $^\ddag$ Dept. of Comp. Sci. \& Tech., TNList Lab; Center for Bio-Inspired Computing Research \\
    State Key Lab for Intell. Tech. \& Systems, Tsinghua University, Beijing, China \\
	\texttt{yangsong@cs.stanford.edu}; \texttt{\{dcszj@, renyong15@mails\}.tsinghua.edu.cn} \\
}

\begin{document}

\maketitle

\begin{abstract}
We propose a vector-valued regression problem whose solution is equivalent to the reproducing kernel Hilbert space (RKHS) embedding of the Bayesian posterior distribution. This equivalence provides a new understanding of kernel Bayesian inference. Moreover, the optimization problem induces a new regularization for the posterior embedding estimator, which is faster and has comparable performance to the squared regularization in kernel Bayes' rule. This regularization coincides with a former thresholding approach used in kernel POMDPs whose consistency remains to be established. Our theoretical work solves this open problem and provides consistency analysis in regression settings. Based on our optimizational formulation, we propose a flexible Bayesian posterior regularization framework which for the first time enables us to put regularization at the distribution level. We apply this method to nonparametric state-space filtering tasks with extremely nonlinear dynamics and show performance gains over all other baselines.
\end{abstract} 
\section{Introduction}
Kernel methods have long been effective in generalizing linear statistical approaches to nonlinear cases by embedding a sample to the reproducing kernel Hilbert space (RKHS)~\cite{smola1998learning}. In recent years, the idea has been generalized to embedding probability distributions~\cite{berlinet2011reproducing,smola2007hilbert}. Such embeddings of probability measures are usually called \emph{kernel embeddings} (\aka \emph{kernel means}). Moreover, \cite{song2009hilbert,fukumizu2011kernel,song2013kernel} show that statistical operations of distributions can be realized in RKHS by manipulating kernel embeddings via linear operators. This approach has been applied to various statistical inference and learning problems, including training hidden Markov models (HMM)~\cite{song2010hilbert}, belief propagation (BP) in tree graphical models~\cite{song2010nonparametric}, planning Markov decision processes (MDP)~\cite{grunewalder2012modelling} and partially observed Markov decision processes (POMDP)~\cite{nishiyama2012hilbert}.

One of the key workhorses in the above applications is the \emph{kernel Bayes' rule}~\cite{fukumizu2011kernel}, which establishes the relation among the RKHS representations of the priors, likelihood functions and posterior distributions. Despite empirical success, the characterization of kernel Bayes' rule remains largely incomplete. For example, it is unclear how the estimators of the posterior distribution embeddings relate to optimizers of some loss functions, though the vanilla Bayes' rule has a nice connection~\cite{Williams:BayesCond1980}. This makes generalizing the results especially difficult and hinters the intuitive understanding of kernel Bayes' rule.

To alleviate this weakness, we propose a vector-valued regression~\cite{micchelli2005learning} problem whose optimizer is the posterior distribution embedding. This new formulation is inspired by the progress in two fields: 1) the alternative characterization of conditional embeddings as regressors~\cite{lever2012conditional}, and 2) the introduction of posterior regularized Bayesian inference (RegBayes)~\cite{zhu2014bayesian} based on an optimizational reformulation of the Bayes' rule.

We demonstrate the novelty of our formulation by providing a new understanding of kernel Bayesian inference, with theoretical, algorithmic and practical implications. On the theoretical side, we are able to prove the (weak) consistency of the estimator obtained by solving the vector-valued regression problem under reasonable assumptions. As a side product, our proof can be applied to a thresholding technique used in \cite{nishiyama2012hilbert}, whose consistency is left as an open problem. On the algorithmic side, we propose a new regularization technique, which is shown to run faster and has comparable accuracy to squared regularization used in the original kernel Bayes' rule~\cite{fukumizu2011kernel}. Similar in spirit to RegBayes, we are also able to derive an extended version of the embeddings by directly imposing regularization on the posterior distributions. We call this new framework kRegBayes. Thanks to RKHS embeddings of distributions, this is the first time, to the best of our knowledge, people can do posterior regularization without invoking linear functionals (such as moments) of the random variables. On the practical side, we demonstrate the efficacy of our methods on both simple and complicated synthetic state-space filtering datasets.

Same to other algorithms based on kernel embeddings, our kernel regularized Bayesian inference framework is nonparametric and general. The algorithm is nonparametric, because the priors, posterior distributions and likelihood functions are all characterized by weighted sums of data samples. Hence it does not need the explicit mechanism such as differential equations of a robot arm in filtering tasks. It is general in terms of being applicable to a broad variety of domains as long as the kernels can be defined, such as strings, orthonormal matrices, permutations and graphs. 
\vspace{-.1cm}
\section{Preliminaries}
\vspace{-.1cm}

\subsection{Kernel embeddings}\vspace{-.1cm}

Let $(\mcal{X},\mcal{B}_\mcal{X})$ be a measurable space of random variables, $p_X$ be the associated probability measure and $\Hx$ be a RKHS with kernel $k(\cdot,\cdot)$. We define the \emph{kernel embedding} of $p_X$ to be $\mu_X = \mbb{E}_{p_X}[\phi(X)] \in \Hx$, where $\phi(X) = k(X,\cdot)$ is the feature map. Such a vector-valued expectation always exists if the kernel is bounded, namely $\sup_x k_\XX(x,x) < \infty$. 

The concept of kernel embeddings has several important statistical merits. Inasmuch as the reproducing property, the expectation of $f \in \mcal{H}$ \wrt $p_X$ can be easily computed as $\mbb{E}_{p_X}[f(X)] = \mbb{E}_{p_X}[\langle f,\phi(X)\rangle] = \langle f,\mu_X\rangle$. There exists \emph{universal kernels}~\cite{micchelli2006universal} whose corresponding RKHS $\mcal{H}$ is dense in $\mcal{C}_\mcal{X}$ in terms of sup norm. This means $\HH$ contains a rich range of functions $f$ and their expectations can be computed by inner products without invoking usually intractable integrals. In addition, the inner product structure of the embedding space $\mcal{H}$ provides a natural way to measure the differences of distributions through norms. 

In much the same way we can define kernel embeddings of linear operators. Let $(\mcal{X},\mcal{B}_\mcal{X})$ and $(\mcal{Y},\mcal{B}_{\mcal{Y}})$ be two measurable spaces, $\phi(x)$ and $\psi(y)$ be the measurable feature maps of corresponding RKHS $\mcal{H}_\mcal{X}$ and $\mcal{H}_\mcal{Y}$ with bounded kernels, and $p$ denote the joint distribution of a random variable $(X,Y)$ on $\mcal{X}\times\mcal{Y}$ with product measures. The \emph{covariance operator} $\mcal{C}_{XY}$ is defined as $\mcal{C}_{XY} = \mbb{E}_p[\phi(X) \otimes \psi(Y)]$, where $\otimes$ denotes the tensor product. Note that it is possible to identify $\mcal{C}_{XY}$ with $\mu_{(XY)}$ in $\mcal{H}_\mcal{X} \otimes \mcal{H}_\mcal{Y}$ with the kernel function $k((x_1,y_1),(x_2,y_2)) = k_\XX(x_1,x_2)k_\YY(y_1,y_2)$~\cite{aronszajn1950theory}. There is an important relation between kernel embeddings of distributions and covariance operators, which is fundamental for the sequel:
\begin{theorem}[\cite{song2009hilbert,fukumizu2011kernel}]\label{thm:found}
Let $\mu_X$, $\mu_Y$ be the kernel embeddings of $p_X$ and $p_Y$ respectively. If $\mcal{C}_{XX}$ is injective, $\mu_{X} \in \mcal{R}(\mcal{C}_{XX})$ and $\mbb{E}[g(Y)\mid X=\cdot] \in \mcal{H}_{\mcal{X}}$ for all $g\in\mcal{H}_{\mcal{Y}}$, then
\begin{equation}
\mu_{Y} = \mcal{C}_{YX}\mcal{C}_{XX}^{-1}\mu_X.
\end{equation}
In addition, $\mu_{Y | X=x}=\mbb{E}[\psi(Y) | X=x] = \mcal{C}_{YX}\mcal{C}_{XX}^{-1}\phi(x)$.
\end{theorem}
On the implementation side, we need to estimate these kernel embeddings via samples. An intuitive estimator for the embedding $\mu_X$ is 
$\widehat{\mu}_X = \frac{1}{N}\sum_{i=1}^N \phi(x_i)$, where $\{x_i\}_{i=1}^N$ is a sample from $p_X$. Similarly, the covariance operators can also be estimated by $\widehat{\mcal{C}}_{XY} = \frac{1}{N}\sum_{i=1}^N \phi(x_i)\otimes\psi(y_i)$. Both operators are shown to converge in the RKHS norm at a rate of $O_p(N^{-\frac{1}{2}})$~\cite{song2009hilbert}.

\vspace{-.1cm}
\subsection{Kernel Bayes' rule}
\vspace{-.1cm}

Let $\pi(Y)$ be the prior distribution of a random variable $Y$, $p(X=x\mid Y)$ be the likelihood, $p^\pi(Y\mid X=x)$ be the posterior distribution given $\pi(Y)$ and observation $x$, and $p^\pi(X,Y)$ be the joint distribution incorporating $\pi(Y)$ and $p(X\mid Y)$. Kernel Bayesian inference aims to obtain the posterior embedding $\mu_Y^\pi(X=x)$ given a prior embedding $\pi_Y$ and a covariance operator $\mcal{C}_{XY}$. By Bayes' rule, $p^\pi(Y\mid X = x) \propto \pi(Y)p(X=x\mid Y) $. We assume that there exists a joint distribution $p$ on $\mcal{X}\times\mcal{Y}$ whose conditional distribution matches $p(X\mid Y)$ and let $\CC_{XY}$ be its covariance operator. Note that we do not require $p = p^\pi$ hence $p$ can be any convenient distribution.

According to \thmref{thm:found}, $\mu_Y^\pi(X=x) = \mcal{C}_{YX}^\pi {\mcal{C}_{XX}^\pi}^{-1} \phi(x)$, where $\mcal{C}_{YX}^\pi$ corresponds to the joint distribution $p^\pi$ and $\mcal{C}_{XX}^\pi$ to the marginal probability of $p^\pi$ on $X$. Recall that $\mcal{C}_{YX}^\pi$ can be identified with $\mu_{(YX)}$ in $\mcal{H}_{\mcal{Y}}\otimes \mcal{H}_{\mcal{X}}$, we can apply \thmref{thm:found} to obtain $\mu_{(YX)} = \mcal{C}_{(YX)Y}\mcal{C}_{YY}^{-1} \pi_Y$, where $\mcal{C}_{(YX)Y}:=\mbb{E}[\psi(Y)\otimes\phi(X)\otimes\psi(Y)]$. Similarly, $\mcal{C}_{XX}^\pi$ can be represented as $\mu_{(XX)} = \mcal{C}_{(XX)Y}\mcal{C}_{YY}^{-1} \pi_Y$. This way of computing posterior embeddings is called the \emph{kernel Bayes' rule}~\cite{fukumizu2011kernel}.

Given estimators of the prior embedding $\widehat{\pi}_Y = \sum_{i=1}^m \tilde{\alpha}_i \psi(y_i)$ and the covariance operator $\widehat{\mcal{C}}_{YX}$, The posterior embedding can be obtained via $\widehat{\mu}_Y^\pi(X=x) = \widehat{\mcal{C}}_{YX}^\pi ([\widehat{\mcal{C}}_{XX}^\pi]^2 + \lambda I)^{-1}\widehat{\mcal{C}}_{XX}^\pi \phi(x)$
, where squared regularization is added to the inversion. Note that the regularization for $\widehat{\mu}_Y^\pi(X=x)$ is not unique. A thresholding alternative is proposed in \cite{nishiyama2012hilbert} without establishing the consistency. We will discuss this thresholding regularization in a different perspective and give consistency results in the sequel.

\vspace{-.1cm}
\subsection{Regularized Bayesian inference}
\vspace{-.1cm}

Regularized Bayesian inference (RegBayes~\cite{zhu2014bayesian}) is based on a variational formulation of the Bayes' rule~\cite{Williams:BayesCond1980}. The posterior distribution can be viewed as the solution of $\min_{p(Y|X=x)} \up{KL}(p(Y|X=x) \| \pi(Y)) - \int \log p(X=x|Y) \ud p(Y|X=x)$, subjected to $p(Y|X=x) \in \mcal{P}_{\up{prob}}$, where $\mcal{P}_{\up{prob}}$ is the set of valid probability measures. RegBayes combines this formulation and posterior regularization~\cite{ganchev2010posterior} in the following way
\begin{gather*}
\min_{p(Y|X=x),\xi} \up{KL}(p(Y|X=x) \| \pi(Y)) - \int \log p(X=x|Y) \ud p(Y|X=x) + U(\xi)\\
s.t.\quad p(Y|X=x) \in \mcal{P}_{\up{prob}}(\xi),\notag
\end{gather*}
where $\mcal{P}_{\up{prob}}(\xi)$ is a subset depending on $\xi$ and $U(\xi)$ is a loss function. Such a formulation makes it possible to regularize Bayesian posterior distributions, smoothing the gap between Bayesian generative models and discriminative models. Related applications include max-margin topic models~\cite{Zhu:medlda} and infinite latent SVMs~\cite{zhu2014bayesian}.

Despite the flexibility of RegBayes, regularization on the posterior distributions is practically imposed indirectly via expectations of a function. We shall see soon in the sequel that our new framework of kernel Regularized Bayesian inference can control the posterior distribution in a direct way.

\vspace{-.1cm}
\subsection{Vector-valued regression}
\vspace{-.1cm}

The main task for vector-valued regression~\cite{micchelli2005learning} is to minimize the following objective
\begin{align*}
E(f) := \sum_{i=1}^n \norm{y_j - f(x_j)}^2_{\Hy} + \lambda \norm{f}^2_{\HK},
\end{align*}
where $y_j \in \Hy$, $f: \mcal{X}\rightarrow \Hy$. Note that $f$ is a function with RKHS values and we assume that $f$ belongs to a \emph{vector-valued} RKHS $\HK$.
In vector-valued RKHS, the kernel function $k$ is generalized to linear operators $\mcal{L}(\Hy) \ni K(x_1,x_2):\Hy \rightarrow \Hy $, such that $K(x_1,x_2)y := (K_{x_2}y)(x_1)$ for every $x_1,x_2\in\XX$ and $y \in \Hy$, where $K_{x_2}y \in \HK$. The reproducing property is generalized to $\langle y,f(x)\rangle_{\Hy} = \langle K_x y,f\rangle_{\HK}$ for every $y\in\Hy$, $f\in\HK$ and $x \in \XX$.
In addition, \cite{micchelli2005learning} shows that the representer theorem still holds for vector-valued RKHS. 

\vspace{-.15cm}
\section{Kernel Bayesian inference as a regression problem}
\vspace{-.15cm}

One of the unique merits of the posterior embedding $\mu_Y^\pi(X=x)$ is that expectations \wrt posterior distributions can be computed via inner products, \ie, $\langle h,\mu_Y^\pi(X=x)\rangle = \mbb{E}_{p^\pi(Y|X=x)}[h(Y)]$ for all $h\in \Hy$. Since $\mu_Y^\pi(X=x) \in \Hy$, $\mu_Y^\pi$ can be viewed as an element of a vector-valued RKHS $\HK$ containing functions $f:\XX\rightarrow \Hy$. 

A natural optimization objective~\cite{lever2012conditional} thus follows from the above observations
\begin{align}
\Ep[\mu] := \sup_{\norm{h}_{\mcal{Y}}\leq 1} \mbb{E}_X\left[ (\mbb{E}_Y[h(Y)|X] - \langle h,\mu(X)\rangle_{\Hy})^2 \right],
\end{align}
where $\mbb{E}_X[\cdot]$ denotes the expectation \wrt $p^\pi(X)$ and $\mbb{E}_Y[\cdot|X]$ denotes the expectation \wrt the Bayesian posterior distribution, \ie, $p^\pi (Y\mid X) \propto \pi(Y)p(X\mid Y)$. Clearly, $\mu_{Y}^\pi = \arginf_{\mu}\Ep[\mu]$. Following \cite{lever2012conditional}, we introduce an upper bound $\Ep_s$ for $\Ep$ by applying Jensen's and Cauchy-Schwarz's inequalities consecutively
\begin{align}
\Ep_s[\mu] := \mbb{E}_{(X,Y)}[\norm{\psi(Y) - \mu(X)}_{\Hy}^2],
\end{align}
where $(X,Y)$ is the random variable on $\cal{X}\times\mcal{Y}$ with the joint distribution $p^\pi(X,Y) = \pi(Y)p(X\mid Y)$.

The first step to make this optimizational framework practical is to find finite sample estimators of $\Ep_s[\mu]$. We will show how to do this in the following section.

\vspace{-.15cm}
\subsection{A consistent estimator of $\Ep_s[\mu]$}
\vspace{-.15cm}

Unlike the conditional embeddings in \cite{lever2012conditional}, we do not have \iid samples from the joint distribution $p^\pi(X,Y)$, as the priors and likelihood functions are represented with samples from different distributions. We will eliminate this problem using a kernel trick, which is one of our main innovations in this paper.

The idea is to use the inner product property of a kernel embedding $\mu_{(X,Y)}$ to represent the expectation $\mbb{E}_{(X,Y)}[\norm{\psi(Y) - \mu(X)}_{\Hy}^2]$ and then use finite sample estimators of $\mu_{(X,Y)}$ to estimate $\Ep_s[\mu]$. Recall that we can identify $\mcal{C}_{XY} := \mbb{E}_{XY}[\phi(X)\otimes\psi(Y)]$ with $\mu_{(X,Y)}$ in a product space $\Hx\otimes\Hy$ with a product kernel $k_\XX k_\YY$ on $\XX\times \YY$~\cite{aronszajn1950theory}. Let $f(x,y) = \norm{\psi(y) - \mu(x)}_{\Hy}^2$ and assume that $f \in \Hx\otimes\Hy$. The optimization objective $\Ep_s[\mu]$ can be written as
\begin{align}
\Ep_s[\mu] = \mbb{E}_{(X,Y)}[\norm{\psi(Y) - \mu(X)}_{\Hy}^2] = \langle f,\mu_{(X,Y)}\rangle_{\Hx\otimes\Hy}.
\end{align}
From \thmref{thm:found}, we assert that $\mu_{(X,Y)} = \CC_{(X,Y)Y}\CC_{YY}^{-1}\pi_Y$ and a natural estimator follows to be $\widehat{\mu}_{(X,Y)} = \widehat{\CC}_{(X,Y)Y}(\widehat{\CC}_{YY}+\lambda I)^{-1}\widehat{\pi}_Y$. As a result, $\widehat{\Ep}_s[\mu]:= \langle \widehat{\mu}_{(X,Y)},f\rangle_{\Hx\otimes\Hy}$ and we introduce the following proposition to write $\widehat{\Ep}_s$ in terms of Gram matrices.
\begin{proposition}[Proof in Appendix]\label{prop:beta}
Suppose $(X,Y)$ is a random variable in $\XX\times\YY$, where the prior for $Y$ is $\pi(Y)$ and the likelihood is $p(X\mid Y)$. Let $\Hx$ be a RKHS with kernel $k_\XX$ and feature map $\phi(x)$, $\Hy$ be a RKHS with kernel $k_\YY$ and feature map $\psi(y)$, $\phi(x,y)$ be the feature map of $\Hx\otimes\Hy$, $\widehat{\pi}_Y = \sum_{i=1}^l \tilde{\alpha}_i \psi(\tilde{y}_i)$ be a consistent estimator of $\pi_Y$ and $\{(x_i,y_i)\}_{i=1}^n$ be a sample representing $p(X\mid Y)$. Under the assumption that $f(x,y) = \norm{\psi(y) - \mu(x)}_{\Hy}^2 \in \Hx\otimes\Hy$, we have
\begin{align}
\widehat{\Ep}_s[\mu] = \sum_{i=1}^{n}\beta_i \norm{\psi(y_i) - \mu(x_i)}_{\Hy}^2,
\end{align}
where $\bs{\beta} = (\beta_1,\cdots,\beta_n)^\T$ is given by $\bs{\beta} = (G_Y + n\lambda I)^{-1} \tilde{G}_Y\tilde{\bs{\alpha}}$,
where $(G_Y)_{ij} = k_\YY(y_i,y_j)$, $(\tilde{G}_Y)_{ij} = k_\YY(y_i,\tilde{y}_j)$, and $\tilde{\bs{\alpha}} = (\tilde{\alpha}_1,\cdots,\tilde{\alpha}_l)^\T$.
\end{proposition}
The consistency of $\widehat{\Ep}_s[\mu]$ is a direct consequence of the following theorem adapted from~\cite{fukumizu2011kernel}, since the Cauchy-Schwarz inequality ensures $|\langle \mu_{(X,Y)},f\rangle - \langle \widehat{\mu}_{(X,Y)},f\rangle|\leq \norm{\mu_{(X,Y)}-\widehat{\mu}_{(X,Y)}}\norm{f}$.
\begin{theorem}[Adapted from \cite{fukumizu2011kernel}, Theorem 8]\label{thm:consist1}
Assume that $\CC_{YY}$ is injective, $\widehat{\pi}_Y$ is a consistent estimator of $\pi_Y$ in $\Hy$ norm, and that $\mbb{E}[k((X,Y),(\tilde{X},\tilde{Y}))\mid Y=y,\tilde{Y}=\tilde{y}]$ is included in $\Hy \otimes \Hy$ as a function of $(y,\tilde{y})$, where $(\tilde{X},\tilde{Y})$ is an independent copy of $(X,Y)$. Then, if the regularization coefficient $\lambda_n$ decays to $0$ sufficiently slowly,
\begin{align}
\norm{\widehat{\CC}_{(X,Y)Y}(\widehat{\CC}_{YY}+\lambda_n I)^{-1} \widehat{\pi}_Y - \mu_{(X,Y)}}_{\Hx\otimes\Hy}\rightarrow 0
\end{align}
in probability as $n \rightarrow \infty$.
\end{theorem}
Although $\widehat{\Ep}_s[\mu]$ is a consistent estimator of $\Ep_s[\mu]$, it does not necessarily have minima, since the coefficients $\beta_i$ can be negative. One of our main contributions in this paper is the discovery that we can ignore data points $(x_i,y_i)$ with a negative $\beta_i$, \ie, replacing $\beta_i$ with $\beta_i^+:=\max(0,\beta_i)$ in $\widehat{\Ep}_s[\mu]$. We will give explanations and theoretical justifications in the next section.

\vspace{-.15cm}
\subsection{The thresholding regularization}
\vspace{-.15cm}

We show in the following theorem that $\Epp := \sum_{i=1}^n \beta_i^+ \norm{\psi(y_i)-\mu(x_i)}^2$ converges to $\Ep_s[\mu]$ in probability in discrete situations. The trick of replacing $\beta_i$ with $\beta_i^+$ is named \emph{thresholding regularization}.
\begin{theorem}[Proof in Appendix]\label{thm:consist2}
	Assume that $\XX$ is compact and $|\YY| < \infty$, $k$ is a strictly positive definite continuous kernel with $\sup_{(x,y)} k((x,y),(x,y)) < \kappa$ and $f(x,y) = \norm{\psi(y) - \mu(x)}_{\Hy}^2 \in \Hx\otimes\Hy$. With the conditions in \thmref{thm:consist1}, we assert that $\widehat{\mu}^+_{(X,Y)}$ is a consistent estimator of $\mu_{(X,Y)}$ and $\left|\Epp - \Ep_s[\mu]\right|\rightarrow 0$ in probability as $n\rightarrow \infty$.
\end{theorem}

In the context of partially observed Markov decision processes (POMDPs)~~\cite{nishiyama2012hilbert}, a similar thresholding approach, combined with normalization, was proposed to make the Bellman operator isotonic and contractive. However, the authors left the consistency of that approach as an open problem. The justification of normalization has been provided in~\cite{lever2012conditional}, Lemma~2.2 under the finite space assumption. A slight modification of our proof of \thmref{thm:consist2} (change the probability space from $\XX\times\YY$ to $\XX$) can complete the other half as a side product, under the same assumptions.

Compared to the original squared regularization used in~\cite{fukumizu2011kernel}, thresholding regularization is more computational efficient because 1) it does not need to multiply the Gram matrix twice, and 2) it does not need to take into consideration those data points with negative $\beta_i$'s. In many cases a large portion of $\{\beta_i\}_{i=1}^n$ is negative but the sum of their absolute values is small. The finite space assumption in \thmref{thm:consist2} may also be weakened, but it requires deeper theoretical analyses.

\vspace{-.1cm}
\subsection{Minimizing $\Epp$}
\vspace{-.1cm}

Following the standard steps of solving a RKHS regression problem, we add a Tikhonov regularization term to $\Epp$ to provide a well-proposed problem,
\begin{align}
\Eln = \sum_{i=1}^n \beta_i^+ \norm{\psi(y_i) - \mu(x_i)}_{\Hy}^2 + \lambda \norm{\mu}_{\HK}^2.
\end{align}
Let $\muln = \argmin_{\mu} \Eln$. Note that $\Eln$ is a vector-valued regression problem, and the representer theorems in vector-valued RKHS apply here. We summarize the matrix expression of $\muln$ in the following proposition.
\begin{proposition}[Proof in Appendix]\label{prop:s1}
Without loss of generality, we assume that $\beta_i^+ \neq 0$ for all $1\leq i\leq n$. Let $\mu \in \HK$ and choose the kernel of $\HK$ to be $K(x_i,x_j) = k_\XX(x_i,x_j)\mcal{I}$, where $\mcal{I}: \HK\rightarrow\HK$ is an identity map.
Then
\begin{equation}
\muln(x) = \Psi(K_X + \lambda_n \Lambda^+)^{-1} K_{:x},
\end{equation}
where $\Psi=(\psi(y_1),\cdots,\psi(y_n))$, $(K_X)_{ij} = k_\XX(x_i,x_j)$, $\Lambda^+ = \up{diag}(1/\beta_1^+,\cdots,1/\beta_n^+)$, $K_{:x} = (k_\XX(x,x_1),\cdots,k_\XX(x,x_n))^\T$ and $\lambda_n$ is a positive regularization constant.
\end{proposition}

\vspace{-.1cm}
\subsection{Theoretical justifications for $\muln$}
\vspace{-.1cm}

In this section, we provide theoretical explanations for using $\muln$ as an estimator of the posterior embedding under specific assumptions. Let $\mu^* = \argmin_{\mu}\Ep[\mu]$, $\mu' = \argmin_{\mu}\Ep_s[\mu]$, and recall that $\muln = \argmin_{\mu}\Eln$. We first show the relations between $\mu^*$ and $\mu'$ and then discuss the relations between $\muln$ and $\mu'$.

The forms of $\Ep$ and $\Ep_s$ are exactly the same for posterior kernel embeddings and conditional kernel embeddings. As a consequence, the following theorem in~\cite{lever2012conditional} still hold.
\begin{theorem}[\cite{lever2012conditional}]
If there exists a $\mu^* \in \HK$ such that for any $h\in\Hy$, $\mbb{E}[h|X] = \langle h,\mu^*(X)\rangle_{\Hy}$ $p_X$-a.s., then $\mu^*$ is the $p_X$-a.s. unique minimiser of both objectives:
\begin{align*}
\mu^* = \argmin_{\mu\in\HK}\Ep[\mu] = \argmin_{\mu\in\HK}\Ep_s[\mu].
\end{align*}
\end{theorem}
This theorem shows that if the vector-valued RKHS $\HK$ is rich enough to contain $\mu_{Y|X=x}^\pi$, both $\Ep$ and $\Ep_s$ can lead us to the correct embedding. In this case, it is reasonable to use $\mu'$ instead of $\mu^*$. For the situation where $\mu_{Y|X=x}^\pi \not\in \HK$, we refer the readers to \cite{lever2012conditional}.

Unfortunately, we cannot obtain the relation between $\muln$ and $\mu'$ by referring to~\cite{caponnetto2007optimal}, as in~\cite{lever2012conditional}. The main difficulty here is that $\{(x_i,y_i)\}|_{i=1}^n$ is not an \iid sample from $p^\pi(X,Y) = \pi(Y)p(X\mid Y)$ and the estimator $\Epp$ does not use \iid samples to estimate expectations. Therefore the concentration inequality (\cite{caponnetto2007optimal}, Prop.~2) used in the proofs of \cite{caponnetto2007optimal} cannot be applied.

To solve the problem, we propose \thmref{thm:conc} (in Appendix) which can lead to a consistency proof for $\muln$. The relation between $\muln$ and $\mu'$ can now be summarized in the following theorem.
\begin{theorem}[Proof in Appendix]\label{thm:diff}
Assume Hypothesis~1 and Hypothesis~2 in \cite{de2005risk} and our Assumption~1 (in the Appendix) hold. With the conditions in \thmref{thm:consist2}, we assert that if $\lambda_n$ decreases to 0 sufficiently slowly, 
\begin{align}
\Ep_s[\widehat{\mu}_{\lambda_n,n}] - \Ep_s[\mu'] \rightarrow 0
\end{align}
in probability as $n\rightarrow \infty$.
\end{theorem} 
\vspace{-.1cm}
\section{Kernel Bayesian inference with posterior regularization}
\vspace{-.1cm}

Based on our optimizational formulation of kernel Bayesian inference, we can add additional regularization terms to control the posterior embeddings. This technique gives us the possibility to incorporate rich side information from domain knowledge and to enforce supervisions on Bayesian inference. We call our framework of imposing posterior regularization \emph{kRegBayes}.

As an example of the framework, we study the following optimization problem
\begin{align}
\mcal{L} := \underbrace{\sum_{i=1}^m \beta_i^+ \norm{\mu(x_i) - \psi(y_i)}_{\Hy}^2 + \lambda \norm{\mu}_{\HK}^2}_{\Eln} + \underbrace{\delta \sum_{i=m+1}^n \norm{\mu(x_i) - \psi(t_i)}_{\Hy}^2}_{\text{The regularization term}},\label{eqn:l3}
\end{align}
where $\{(x_i,y_i)\}_{i=1}^m$ is the sample used for representing the likelihood, $\{(x_i,t_i)\}_{i=m+1}^n$ is the sample used for posterior regularization and $\lambda, \delta$ are the regularization constants. Note that in RKHS embeddings, $\psi(t)$ is identified as a point distribution at $t$~\cite{berlinet2011reproducing}. Hence the regularization term in \eqref{eqn:l3} encourages the posterior distributions
$p(Y\mid X=x_i)$ to be concentrated at $t_i$. More complicated regularization terms are also possible, such as $\|\mu(x_i) - \sum_{i=1}^l \alpha_i \psi(t_i)\|_{\Hy}$.

Compared to vanilla RegBayes, our kernel counterpart has several obvious advantages. First, the difference between two distributions can be naturally measured by RKHS norms. This makes it possible to regularize the posterior distribution as a whole, rather than through expectations of discriminant functions. Second, the framework of kernel Bayesian inference is totally nonparametric, where the priors and likelihood functions are all represented by respective samples. We will further demonstrate the properties of kRegBayes through experiments in the next section.

Let $\widehat{\mu}_{reg} = \argmin_{\mu}\mcal{L}$. It is clear that solving $\mcal{L}$ is substantially the same as $\Eln$ and we summarize it in the following proposition.
\begin{proposition}\label{prop:s2}
With the conditions in \propref{prop:s1}, we have
\begin{equation}
\widehat{\mu}_{reg}(x) = \Psi(K_X + \lambda \Lambda^+)^{-1} K_{:x},
\end{equation}
where $\Psi=(\psi(y_1),\cdots,\psi(y_n))$, $(K_X)_{ij} = k_\XX(x_i,x_j)|_{1\leq i,j\leq n}$, $\Lambda^+ = \up{diag}(1/\beta_1^+,\cdots,1/\beta_m^+, 1/\delta,\cdots,1/\delta)$, and $K_{:x} = (k_\XX(x,x_1),\cdots,k_\XX(x,x_n))^\T$.
\end{proposition}

\section{Experiments}
In this section, we compare the results of kRegBayes and several other baselines for two state-space filtering tasks. The mechanism behind kernel filtering is stated in \cite{fukumizu2011kernel} and we provide a detailed introduction in Appendix, including all the formula used in implementation.
\paragraph{Toy dynamics}
This experiment is a twist of that used in~\cite{fukumizu2011kernel}. We report the results of extended Kalman filter (EKF)~\cite{julier1997new} and unscented Kalman filter (UKF)~\cite{wan2000unscented}, kernel Bayes'  rule (KBR)~\cite{fukumizu2011kernel}, kernel Bayesian learning with thresholding regularization (pKBR) and kRegBayes.

The data points $\{(\theta_t,x_t,y_t)\}$ are generated from the dynamics
\begin{equation}
\theta_{t+1} = \theta_{t} + 0.4 + \xi_t \quad(\up{mod~} 2\pi),\quad \begin{pmatrix}
x_{t+1}\\
y_{t+1}
\end{pmatrix} = (1+\sin(8\theta_{t+1}))\begin{pmatrix}
\cos \theta_{t+1}\\
\sin \theta_{t+1}
\end{pmatrix} + \zeta_{t},
\end{equation}
where $\theta_t$ is the hidden state, $(x_t,y_t)$ is the observation, $\xi_t \sim \mcal{N}(0,0.04)$ and $\zeta_t \sim \mcal{N}(0,0.04)$. Note that this dynamics is nonlinear for both transition and observation functions. The observation model is an oscillation around the unit circle. There are 1000 training data and 200 validation/test data for each algorithm.

\begin{wrapfigure}{r}{0.45\textwidth}\vspace{-.6cm}
\centering
    \label{fig:exp1_result}\includegraphics[width=0.4\textwidth]{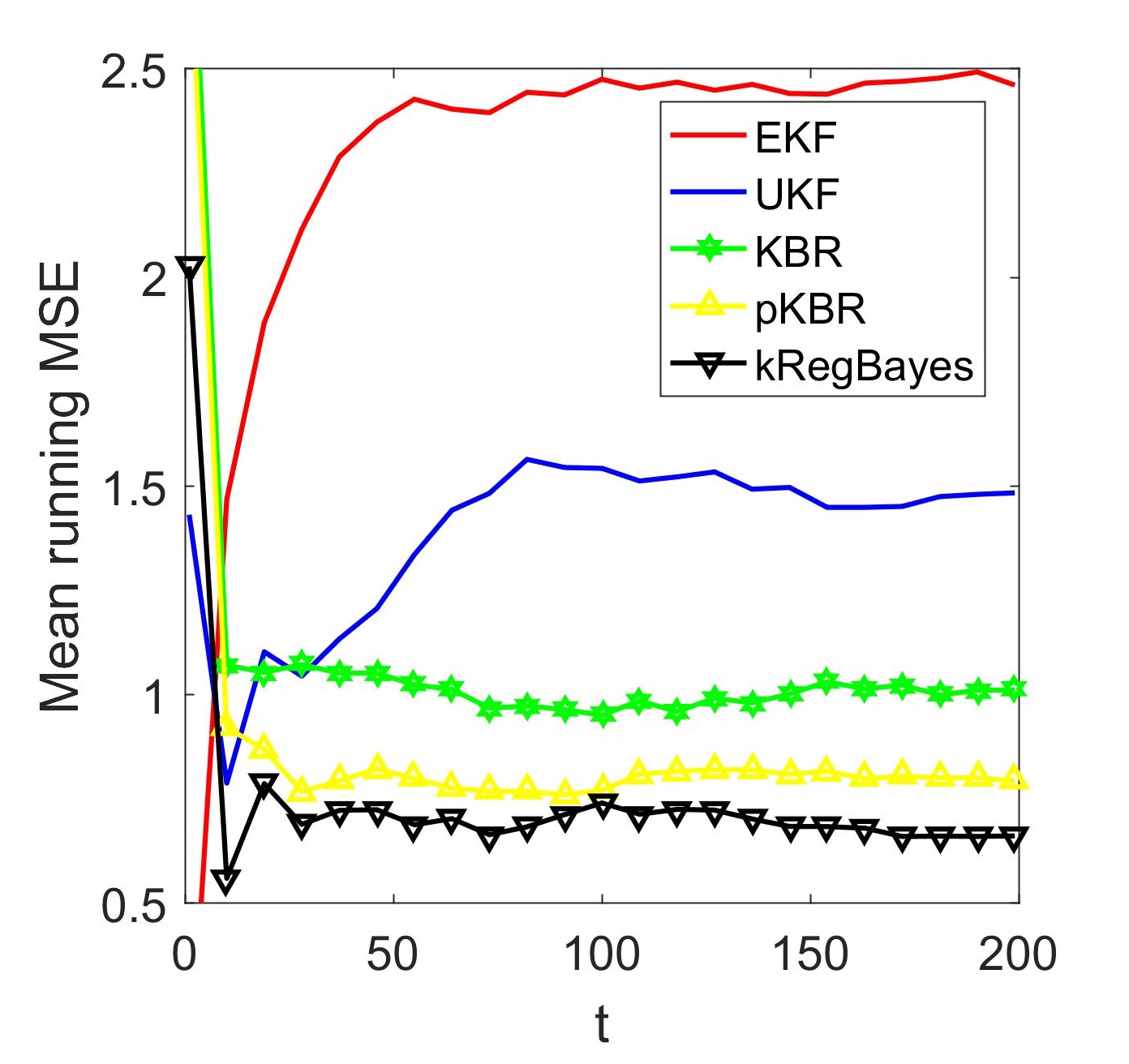}\vspace{-.2cm}
  \caption{Mean running MSEs against time steps for each algorithm. (Best view in color)}\vspace{-.3cm}
\end{wrapfigure}

We suppose that EKF, UKF and kRegBayes know the true dynamics of the model and the first hidden state $\theta_1$. In this case, we use $\tilde{\theta}_{t+1} = \theta_1 + 0.4t~(\up{mod}~2\pi)$ and $(\tilde{x}_{t+1},\tilde{y}_{t+1})^\T = (1+\sin(8\tilde{\theta}_{t+1}))(\cos\tilde{\theta}_{t+1},\sin\tilde{\theta}_{t+1})^\T$ as the supervision data point for the $(t+1)$-th step. We follow \cite{fukumizu2011kernel} to set our parameters.

The results are summarized in \figref{fig:exp1_result}. pKBR has lower errors compared to KBR, which means the thresholding regularization is practically no worse than the original squared regularization. The lower MSE of kRegBayes compared with pKBR shows that the posterior regularization successfully incorporates information from equations of the dynamics. Moreover, pKBR and kRegBayes run faster than KBR. The total running times for 50 random datasets of pKBR, kRegBayes and KBR are respectively 601.3s, 677.5s and 3667.4s.
\paragraph{Camera position recovery}
In this experiment, we build a scene containing a table and a chair, which is derived from \verb|classchair.pov|~(\url{http://www.oyonale.com}). With a fixed focal point, the position of the camera uniquely determines the view of the scene. The task of this experiment is to estimate the position of the camera given the image. This is a problem with practical applications in remote sensing and robotics.

 We vary the position of the camera in a plane with a fixed height. The transition equations of the hidden states are
\begin{equation*}
\theta_{t+1} = \theta_{t} + 0.2 + \xi_{\theta},\quad r_{t+1} = \max(R_2, \min(R_1,r_{t} + \xi_r)), \quad x_{t+1} = \cos\theta_{t+1},\quad y_{t+1} = \sin\theta_{t+1},
\end{equation*}
where $\xi_{\theta} \sim \mcal{N}(0,4e-4)$, $\xi_{r}\sim\mcal{N}(0,1)$, $0\leq R_1<R_2$ are two constants and $\{(x_{t},y_{t})\}|_{t=1}^m$ are treated as the hidden variables. As the observation at $t$-th step, we render a $100\times100$ image with the camera located at $(x_t,y_t)$. For training data, we set $R_1 =0$ and $R_2=10$ while for validation data and test data we set $R_1 = 5$ and $R_2 = 7$. The motivation is to distinguish the efficacy of enforcing the posterior distribution to concentrate around distance 6 by kRegBayes. We show a sample set of training and test images in \figref{fig:exp2_data}.

\begin{figure}
\centering
\includegraphics[width =0.0625\textwidth]{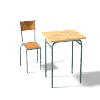}\includegraphics[width =0.0625\textwidth]{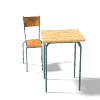}\includegraphics[width =0.0625\textwidth]{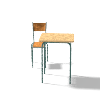}\includegraphics[width =0.0625\textwidth]{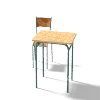}\includegraphics[width =0.0625\textwidth]{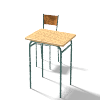}\includegraphics[width =0.0625\textwidth]{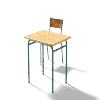}\includegraphics[width =0.0625\textwidth]{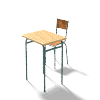}\includegraphics[width =0.0625\textwidth]{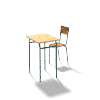}\includegraphics[width =0.0625\textwidth]{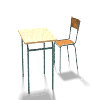}\includegraphics[width =0.0625\textwidth]{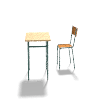}\includegraphics[width =0.0625\textwidth]{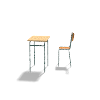}\includegraphics[width =0.0625\textwidth]{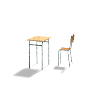}\includegraphics[width =0.0625\textwidth]{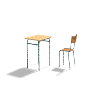}\includegraphics[width =0.0625\textwidth]{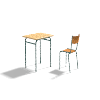}\includegraphics[width =0.0625\textwidth]{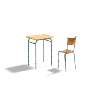}\includegraphics[width =0.0625\textwidth]{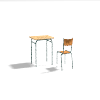}\\
\includegraphics[width =0.0625\textwidth]{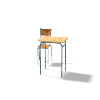}\includegraphics[width =0.0625\textwidth]{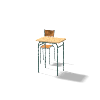}\includegraphics[width =0.0625\textwidth]{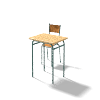}\includegraphics[width =0.0625\textwidth]{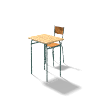}\includegraphics[width =0.0625\textwidth]{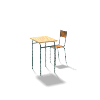}\includegraphics[width =0.0625\textwidth]{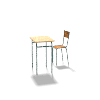}\includegraphics[width =0.0625\textwidth]{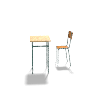}\includegraphics[width =0.0625\textwidth]{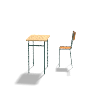}\includegraphics[width =0.0625\textwidth]{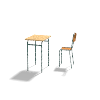}\includegraphics[width =0.0625\textwidth]{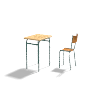}\includegraphics[width =0.0625\textwidth]{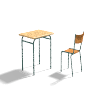}\includegraphics[width =0.0625\textwidth]{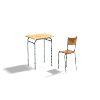}\includegraphics[width =0.0625\textwidth]{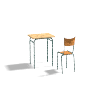}\includegraphics[width =0.0625\textwidth]{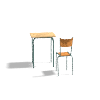}\includegraphics[width =0.0625\textwidth]{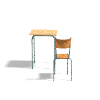}\includegraphics[width =0.0625\textwidth]{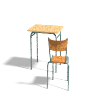}

\caption{First several frames of training data (upper row) and test data (lower row).}\label{fig:exp2_data}
\end{figure}
\begin{figure}
\centering
\subfigure[]{
\label{fig:exp2_result}\includegraphics[width=0.45\textwidth]{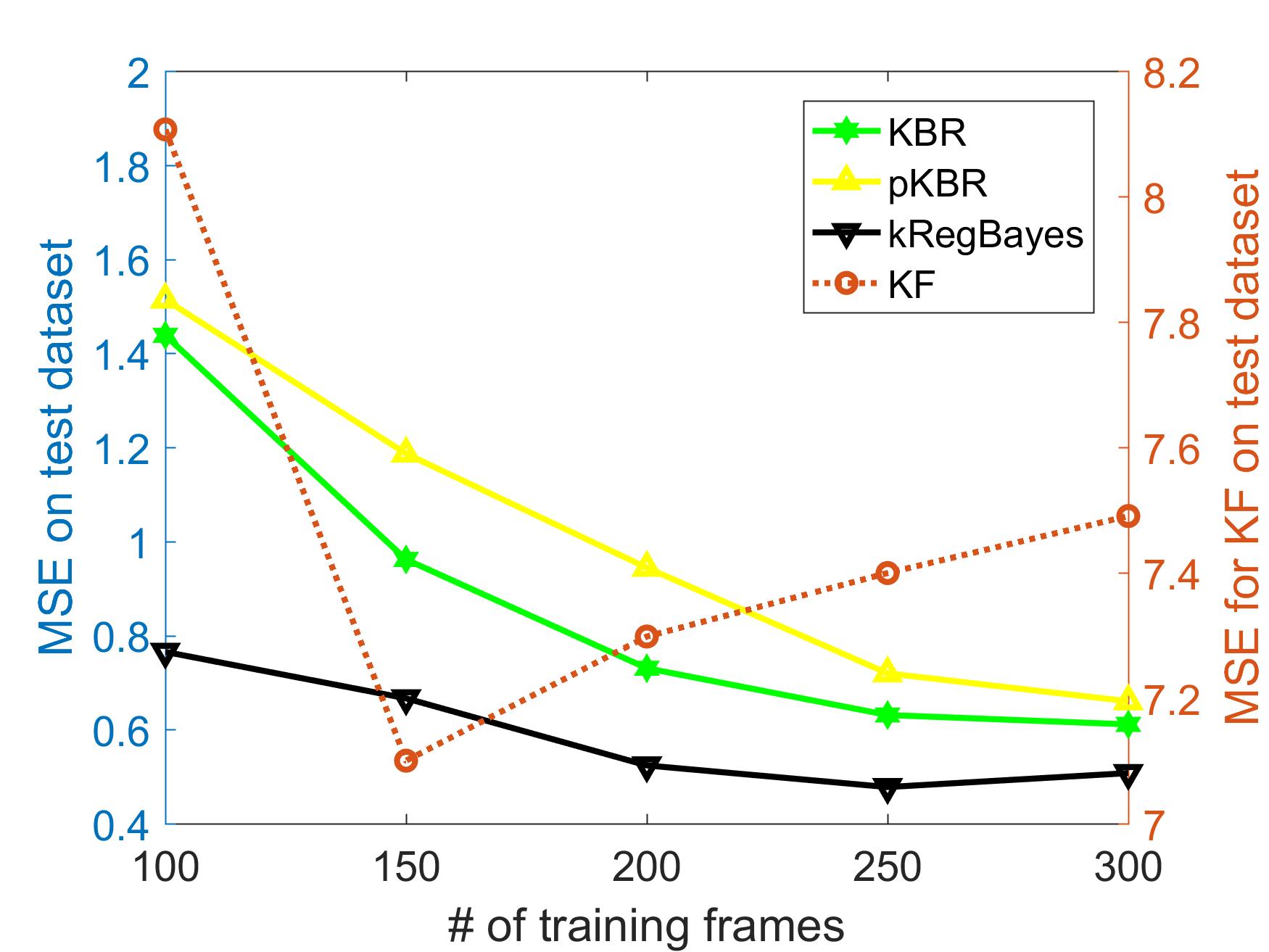}
}
\subfigure[]{
\label{fig:exp2_d}\includegraphics[width=0.45\textwidth]{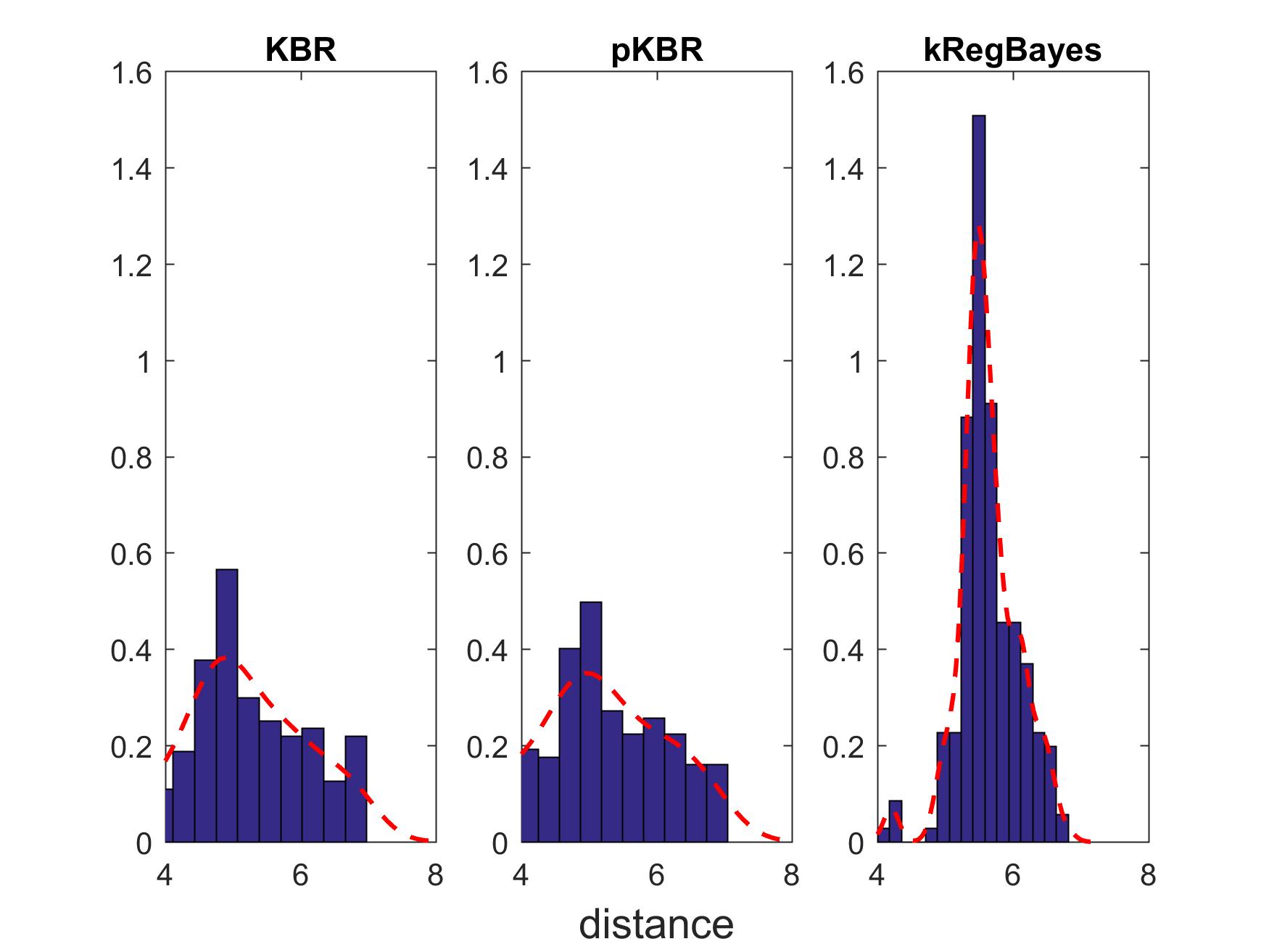}
}
\caption{\subref{fig:exp2_result} MSEs for different algorithms (best view in color). Since KF performs much worse than kernel filters, we use a different scale and plot it on the right $y$-axis. \subref{fig:exp2_d} Probability histograms for the distance between each state and the scene center. All algorithms use 100 training data.}\label{fig:exp2}
\end{figure}

We compare KBR, pKBR and kRegBayes with the traditional linear Kalman filter (KF~\cite{kalman1960new}). Following \cite{song2009hilbert} we down-sample the images and train a linear regressor for observation model. In all experiments, we flatten the images to a column vector and apply Gaussian RBF kernels if needed. The kernel band widths are set to be the median distances in the training data. Based on experiments on the validation dataset, we set $\lambda_T = 1e-6 = 2\delta_T$ and $\mu_T = 1e-5$.

To provide supervision for kRegBayes, we uniformly generate 2000 data points $\{(\hat{x}_i,\hat{y}_t)\}_{i=1}^{2000}$ on the circle $r = 6$. Given the previous estimate $(\tilde{x}_t,\tilde{y}_t)$, we first compute $\hat{\theta}_t = \arctan (\hat{y}_t/\hat{x}_t )$ (where the value $\hat{\theta}_t$ is adapted according to the quadrant of $(\hat{x}_t,\hat{y}_t)$) and estimate $(\breve{x}_{t+1},\breve{y}_{t+1}) = (\cos(\hat{\theta}_t+0.4),\sin(\hat{\theta}_t+0.4))$. Next, we find the nearest point to $(\breve{x}_{t+1},\breve{y}_{t+1})$ in the supervision set $(\tilde{x}_k,\tilde{y}_k)$ and add the regularization $\mu_T\norm{\mu(\mcal{I}_{t+1}) - \phi(\tilde{x}_k,\tilde{y}_k)}$ to the posterior embedding, where $\mcal{I}_{t+1}$ denotes the $(t+1)$-th image.

 We vary the size of training dataset from 100 to 300 and report the results of KBR, pKBR, kRegBayes and KF on 200 test images in \figref{fig:exp2}. KF performs much worse than all three kernel filters due to the extreme non-linearity. The result of pKBR is a little worse than that of KBR, but the gap decreases as the training dataset becomes larger. kRegBayes always performs the best. Note that the advantage becomes less obvious as more data come. This is because kernel methods can learn the distance relation better with more data, and posterior regularization tends to be more useful when data are not abundant and domain knowledge matters. Furthermore, \figref{fig:exp2_d} shows that the posterior regularization helps the distances to concentrate.

\section{Conclusions}
We propose an optimizational framework for kernel Bayesian inference.  With thresholding regularization, the minimizer of the framework is shown to be a reasonable estimator of the posterior kernel embedding. In addition, we propose a posterior regularized kernel Bayesian inference framework called kRegBayes. These frameworks are applied to non-linear state-space filtering tasks and the results of different algorithms are compared extensively.

\section*{Acknowledgements}
We thank all the anonymous reviewers for valuable suggestions. The work was supported by the National Basic Research Program (973 Program) of China (No. 2013CB329403),
National NSF of China Projects (Nos. 61620106010, 61322308, 61332007), the Youth Top-notch Talent Support Program, and Tsinghua Initiative Scientific Research Program (No. 20141080934).

{\small
\bibliography{krb.bib}
}
\newpage
\appendix
\section{Appendix}

\subsection{Kernel filtering}
We first review how to use kernel techniques to do state-space filtering~\cite{fukumizu2011kernel}. Assume that a sample $(y_1,x_1,\cdots,y_{T+1},x_{T+1})$ is given, in which $y_i\in\YY$ is the state and $x_i \in\XX$ is the corresponding observation. The transition and observation probabilities are estimated empirically in a nonparametric way:
\begin{align*}
\hC_{YY_+} = \frac{1}{T}\sum_{i=1}^T \psi(y_i)\otimes\psi(y_{i+1}),\quad \hC_{YX} = \frac{1}{T} \sum_{i=1}^T \psi(y_i) \otimes \phi(x_i).
\end{align*}
The filtering task is composed of two steps. The first step is to predict the next state based on current state, \ie, $p(Y_{t+1}\mid X_1,\cdots,X_t) = \int p(Y_{t+1}\mid Y_{t})p(Y_{t}\mid X_1,\cdots,X_t)d Y_{t}$. The second step is to update the state based on a new observation $x_{t+1}$ via Bayes' rule, \ie, $p(Y_{t+1}\mid X_1,\cdots,X_{t+1})\propto p(Y_{t+1}\mid X_1,\cdots,X_t)p(X_{t+1}\mid Y_{t+1})$. Following these two steps, we can obtain a recursive kernel update formula under different assumptions of the forms of kernel embedding $\widehat{m}_{y_t\mid x_1,\cdots,x_t}$.

For kernel embeddings without posterior regularization, we suppose $\widehat{m}_{y_t\mid x_1,\cdots,x_t} = \sum_{i=1}^T \alpha_i^{(t)}\psi(y_i)$. According to \thmref{thm:found}, the prediction step is realized by $\widehat{m}_{y_{t+1}\mid x_1,\cdots,x_t} = \hC_{Y_+Y}(\hC_{YY} +\lambda_T I)^{-1} \widehat{m}_{y_t\mid x_1,\cdots,x_t} = \Psi_+ (G_Y + T\lambda_T I)^{-1} G_Y \bs{\alpha}^{(t)}$, where $\Psi_+ = (\psi(y_2),\cdots,\psi(y_{T+1}))$, $G_Y$ is the Gram matrix of $\{y_1,\cdots,y_T\}$ and $\bs{\alpha}^{(t)}$ is the vector of coefficients. The update step can be realized by invoking \propref{prop:s1}, \ie, $\widehat{m}_{y_{t+1}\mid x_1,\cdots,x_{t+1}} = \Psi(K_X + \delta_T \Lambda^+)^{-1} K_{:x_{t+1}}$, where $K_X$ is the Gram matrix for $(x_1,\cdots,x_t)$, $\Lambda^+ = \up{diag}(1/\bs{\beta^+})$ and $\bs{\beta} = (G_Y + T\lambda_T I)^{-1} G_{YY_+} (G_Y+T\lambda_T I)^{-1}G_Y \bs{\alpha}^{(t)}$, where $(G_{YY_+})_{ij} = k_\YY(y_i,y_{i+1})$. The update formula of $\bs{\alpha}^{(t+1)}$ can then be summarized as follows
\begin{align}
\bs{\alpha}^{(t+1)} = (K_X + \delta_T\Lambda^+)^{-1} K_{:x_{t+1}}.
\end{align}

For kernel embeddings with posterior regularization, we suppose that for each step $t$, the regularization $\mu_T \norm{\mu(\tilde{x}_t)- \psi(\tilde{y}_t)}$ is used, meaning that $p(Y_{t}|X_1,\cdots,X_t = \tilde{x}_t)$ is encouraged to concentrate around $\delta(Y_{t} = \tilde{y}_t)$. To obtain a recursive formula, we assume that $\widehat{m}_{y_{t}\mid x_1,\cdots,x_t} = \sum_{i=1}^T \alpha_i^{(t)} \psi(y_i) + \sum_{i=1}^N \tilde{\alpha}_i^{(t)} \psi(\tilde{y}_i)$, where $N$ is the number of supervision data points $(\tilde{x}_i,\tilde{y}_i)$. Following a similar logic except replacing \propref{prop:s1} with \propref{prop:s2}, we get the update rule for $\bs{\alpha}^{(t+1)}$ and $\bs{\tilde{\alpha}}^{(t+1)}$
\begin{align}
\bs{\gamma} &= (K_X' + \delta_T \Lambda^\dagger)^{-1} K_{:x_{t+1}}'\\
\bs{\alpha}^{(t+1)} &= \bs{\gamma}[1:m]\\
\bs{\tilde{\alpha}}^{(t+1)} &= (0,\cdots,\bs{\gamma}[m+1],0,\cdots)^\T,
\end{align}
where $\Lambda^\dagger = \up{diag}(1/\bs{\beta}^+, 1/\mu_T)$, $\bs{\beta} = (G_Y + T\lambda_T I)^{-1} G_{YY+}(G_YY + T\lambda_T I)^{-1} (G_{YY}\bs{\alpha}^{(t)} + G_{Y\tilde{Y}}\bs{\tilde{\alpha}}^{(t)})$. $K_X'$ and $K_{:x_{t+1}}'$ are augmented Gram matrices, which incorporate $(\tilde{x}_i,\tilde{y}_i)$. The position of $\bs{\gamma}[m+1]$ in $\bs{\tilde{\alpha}}^{(t+1)}$ corresponds to the index of supervision $(\tilde{x}_k,\tilde{y}_k)$ at $t+1$ step in $\{(\tilde{x}_i,\tilde{y}_i)\}_{i=m+1}^n$.

To obtain $\bs{\alpha}^{(1)}$, we use conditional operators~\cite{song2009hilbert} to estimate $m_{y_1}$ without priors. We set $\bs{\alpha}^{(1)} = (K_X + T\lambda_T I)^{-1}K_{:x_1}$ for both types of kernel filtering and $\bs{\tilde{\alpha}}^{(1)} = \bs{0}$. To decode the state from kernel embeddings, we solve an optimization problem $\hat{y}_t = \argmin_y\norm{m(x) - \psi(y)}$, which can be computed using an iteration scheme as depicted in~\cite{song2009hilbert}.

\subsection{Proofs}
\begin{customprop}{\ref{prop:beta}}
Suppose $(X,Y)$ is a random variable in $\XX\times\YY$, where the prior for $Y$ is $\pi(Y)$ and the likelihood is $p(X\mid Y)$. Let $\Hx$ be a RKHS with kernel $k_\XX$ and feature map $\phi(x)$, $\Hy$ be a RKHS with kernel $k_\YY$ and feature map $\psi(y)$, $\phi(x,y)$ be the feature map of $\Hx\otimes\Hy$, $\widehat{\pi}_Y = \sum_{i=1}^l \tilde{\alpha}_i \psi(\tilde{y}_i)$ be an estimator for $\pi_Y$ and $\{(x_i,y_i)\}_{i=1}^n$ be a sample representing $p(X\mid Y)$. Under the assumption that $f(x,y) = \norm{\psi(y) - \mu(x)}_{\Hy}^2 \in \Hx\otimes\Hy$, we have
\begin{align}
\widehat{\Ep}_s[\mu] = \sum_{i=1}^{n}\beta_i \norm{\psi(y_i) - \mu(x_i)}_{\Hy}^2,
\end{align}
where $\bs{\beta} = (\beta_1,\cdots,\beta_n)^\T$ is given by $\bs{\beta} = (G_Y + n\lambda I)^{-1} \tilde{G}_Y\tilde{\bs{\alpha}}$,
where $(G_Y)_{ij} = k_\YY(y_i,y_j)$, $(\tilde{G}_Y)_{ij} = k_\YY(y_i,\tilde{y}_j)$, and $\tilde{\bs{\alpha}} = (\tilde{\alpha}_1,\cdots,\tilde{\alpha}_l)^\T$.
\end{customprop}
\begin{proof}
The reasoning is similar to \cite{fukumizu2011kernel}, Prop.~5. We only need to show that $\widehat{\mu}_{(X,Y)} = \Phi_{X,Y}\bs{\beta} = \Phi_{X,Y}(G_Y + n\lambda I)^{-1} \tilde{G}_Y\tilde{\bs{\alpha}}$, where $\Phi_{X,Y} = (\phi(x_1,y_1),\cdots,\phi(x_n,y_n))$. Recall that $\widehat{\mu}_{(X,Y)} = \widehat{\CC}_{(X,Y)Y}(\widehat{\CC}_{YY}+\lambda I)^{-1}\widehat{\pi}_Y$. Let $h = (\widehat{\CC}_{YY}+\lambda I)^{-1}\widehat{\pi}_Y$ and decompose it as $h = \sum_{i=1}^n a_i \psi(y_i) + h_\perp$, where $h_\perp$ is perpendicular to $\up{span} \{\psi(y_1),\cdots,\psi(y_n)\}$. Expanding $(\widehat{\CC}_{YY}+\lambda I) h = \widehat{\pi}_Y$, we obtain
\begin{align}
\frac{1}{n}\sum_{i,j\leq n} a_i k_\YY(y_i,y_j)\psi(y_j) + \lambda(\sum_{i\leq n} a_i \psi(y_i) + h_\perp) = \sum_{i\leq l}\tilde{\alpha}_i\psi(\tilde{y}_i).
\end{align}
Multiplying both sides with $\psi(y_k)|_{k=1}^n$, we get $\frac{1}{n}G_Y^2 \bs{a} + \lambda G_Y\bs{a} = \tilde{G}_Y \tilde{\bs{\alpha}}$. Therefore $\widehat{\mu}_{(X,Y)}$ can be written as $\widehat{\mu}_{(X,Y)} = \frac{1}{n}[\sum_{i\leq n} \phi(x_i,y_i)\otimes \psi(y_i)]h = \frac{1}{n}\Phi_{X,Y}G_Y\bs{a} = \Phi_{X,Y}(G_Y+n\lambda I)^{-1}\tilde{G}_Y\tilde{\bs{\alpha}}$.
\end{proof}
\begin{customprop}{\ref{prop:s1}}
Without loss of generality, we assume that $\beta_i^+ \neq 0$ for all $1\leq i\leq n$. Let $\mu \in \HK$ and choose the kernel of $\HK$ to be $K(x_i,x_j) = k_\XX(x_i,x_j)\mcal{I}$, where $\mcal{I}: \HK\rightarrow\HK$ is an identity map.
Then
\begin{equation}
\muln(x) = \Psi(K_X + \lambda_n \Lambda^+)^{-1} K_{:x},
\end{equation}
where $\Psi=(\psi(y_1),\cdots,\psi(y_n))$, $(K_X)_{ij} = k_\XX(x_i,x_j)$, $\Lambda^+ = \up{diag}(1/\beta_1^+,\cdots,1/\beta_n^+)$, $K_{:x} = (k_\XX(x,x_1),\cdots,k_\XX(x,x_n))^\T$ and $\lambda_n$ is a positive regularization constant.
\end{customprop}
\begin{proof}
If $\beta_i^+ = 0$ for any $i$, we can discard the data point $(x_i,y_i)$ without affecting results. Let $\mu = \mu_0 + g$, where $\mu_0 = \sum_{i=1}^n K_{x_i}c_i$. Plugging $\mu = \mu_0 + g$ into $\Eln$ and expand, we obtain $\Eln = \sum_{i=1}^n \beta_i^+ \norm{\psi(y_i)-\mu_0(x_i)}^2 + \lambda_n\norm{\mu_0}^2 + \sum_{i=1}^n \beta_i^+ \norm{g(x_i)}^2 + \lambda_n \norm{g}^2 + 2\lambda_n \langle \mu_0,g\rangle - 2 \sum_{i=1}^n \beta_i^+ \langle g(x_i), \psi(y_i) - \mu_0(x_i)\rangle$.

We conjecture that $\psi(y_i) - \sum_{j=1}^n k_\XX(x_i,x_j) c_j = \frac{\lambda_n}{\beta_i^+} c_i$, for all $1\leq i\leq n$. Actually, substituting these equations into $\Eln$ gives the relation $\lambda_n \langle \mu_0,g\rangle - \sum_{i=1}^n \beta_i^+ \langle g(x_i),\psi(y_i)-\mu_0(x_i)\rangle = 0$. As a result, $\Eln = \widehat{\mcal{E}}_{\lambda,n}[\mu_0] + \sum_{i=1}^n \beta_i^+ \norm{g(x_i)}^2 + \lambda_n \norm{g}^2 \geq \widehat{\mcal{E}}_{\lambda,n}[\mu_0]$, which means that $\mu_0 = \sum_{i=1}^n K_{x_i}c_i$ with $c_i$ satisfying the conjectured equations is the solution. The equation $\psi(y_i) - \sum_{j=1}^n k_\XX(x_i,x_j) c_j = \frac{\lambda_n}{\beta_i^+} c_i$ implies that $(K_X + \lambda_n \Lambda^+)c = \Psi$ and $\mu_0(x) = \sum_{i=1}^n k_\XX(x,x_i)c_i = \Psi(K_X + \lambda_n \Lambda^+)^{-1} K_{:x}$.
\end{proof}

\begin{theorem}\label{thm:finiteconsist}
Assume that $|\XX\times\YY| < \infty$, $k$ is strictly positive definite with $\sup_{(x,y)} k((x,y),(x,y)) < \kappa$ and $f(x,y) = \norm{\psi(y) - \mu(x)}_{\Hy}^2 \in \Hx\otimes\Hy$. With the conditions in \thmref{thm:consist1}, we assert that $\widehat{\mu}^+_{(X,Y)}$ is a consistent estimator of $\mu_{(X,Y)}$ and $\left|\Epp - \Ep_s[\mu]\right|\rightarrow 0$ in probability as $n\rightarrow \infty$.
\end{theorem}
\begin{proof}
We only need to show that $\mup := \sum_{i=1}^n \beta_i^+ \phi(x_i)\otimes\psi(y_i)$ converges to $\mu_{(X,Y)}$ in probability as $n\rightarrow \infty$, since $\left|\Epp - \Ep_s[\mu]\right| = |\langle f, \mup - \mu_{(X,Y)}\rangle| \leq \norm{f}\norm{\mup - \mu_{(X,Y)}}$. From \thmref{thm:consist1} we know that $\widehat{\mu}_{(X,Y)}$ converges to $\mu_{(X,Y)}$ in probability, hence it is sufficient to show that $\mup$ converges to $\widehat{\mu}_{(X,Y)}$ in RKHS norm as $n\rightarrow \infty$.

Let $|\XX\times\YY| = M$. Without losing generality, we assume $\XX\times\YY = \{ (x_1,y_1),\cdots,(x_M,y_M) \}$ and $\{(x_1,y_1),\cdots,(x_n,y_n)\}$ is a sample representing $p(X\mid Y)$. According to Theorem 4 in~\cite{smola2003kernels}, $k$ is strictly positive definite on a finite set implies that $\Hx\otimes\Hy$ consists of all bounded functions on $\XX\times\YY$. In particular, $\Hx\otimes\Hy$ contains the function
\begin{equation}
g(x_i,y_i) = \begin{cases}
1,\quad \beta_i < 0\\
0,\quad \up{otherwise}.
\end{cases}
\end{equation}
We denote $b:= \max_g\norm{g}_{\Hx\otimes\Hy} = \max_\mbf{g} \mbf{g}^\T K^{-1}\mbf{g}$ for all possibilities of $\bs{\beta}$. Here $\mbf{g}$ represents the point evaluations of $g$ on $\{(x_i,y_i)\}_{i=1}^M$ and $K_{ij}|_{1\leq i,j\leq M} = k((x_i,y_i),(x_j,y_j))$. Note that $g(x,y)$ is non-negative, thus $\mbb{E}[g(X,Y)] = \langle g,\mu_{(X,Y)}\rangle\geq 0$. For sufficiently large $n$, $|\langle g,\widehat{\mu}_{(X,Y)} - \mu_{(X,Y)}\rangle|\leq \norm{g}\norm{\widehat{\mu}_{(X,Y)} - \mu_{(X,Y)}} \leq \epsilon b$ in arbitrarily high probability. In this case $\langle g,\widehat{\mu}_{(X,Y)}\rangle = -\sum_{i=1}^n \beta_i^- \geq - \epsilon b$, where $\beta_i^- = -\min(0,\beta_i)$, and $\norm{\mup - \widehat{\mu}_{(X,Y)}} = \norm{\sum_{i=1}^n \beta_i^- \phi(x_i,y_i)} = \sqrt{\sum_{i,j}\beta_i^- \beta_j^- k((x_i,y_i),(x_j,y_j))}\leq \sqrt{\kappa}\sum_{i=1}\beta_i^- \leq \epsilon b\sqrt{\kappa}$. The inequalities can now be linked and the theorem proved.
\end{proof}
\begin{theorem}\label{thm:betas}
	Assume that $|\XX\times\YY| < \infty$, $k$ is strictly positive definite with $\sup_{(x,y)} k((x,y),(x,y)) < \kappa$, we assert $\sum_{i=1}^n\beta_i^+ \rightarrow 1$ in probability as $n\rightarrow \infty$.
\end{theorem}
\begin{proof}
	The proof follows a similar reasoning to that in \thmref{thm:finiteconsist}. Let $|\XX\times\YY| = M$ and $\{(x_1,y_1),\cdots,(x_n,y_n)\}$ be a sample representing $p(X\mid Y)$. According to Theorem 4 in~\cite{smola2003kernels}, $k$ is strictly positive definite on a finite set implies that $\Hx\otimes\Hy$ consists of all bounded functions on $\XX\times\YY$. In particular, $\Hx\otimes\Hy$ contains the function $f(x,y) \equiv 1$. From \thmref{thm:finiteconsist} we know that $\widehat{\mu}^+ = \sum_{i=1}^n \beta_i^+ \phi(x_i)\otimes\psi(y_i) \rightarrow \mu $ in probability. Therefore, $|\sum_{i=1}^n \beta_i^+ - 1|= |\langle f,\widehat{\mu}^+_{(X,Y)} -\mu_{(X,Y)} \rangle| \leq \norm{f}\norm{\widehat{\mu}^+_{(X,Y)} - \mu_{(X,Y)}} \rightarrow 0$ in probability.
\end{proof}
Since $\beta_i$'s do not depend on $X_1,\cdots,X_n$, we have the following corollary:
\begin{corollary}\label{col:betapsum}
	Assume that $|\YY| < \infty$, $k$ is strictly positive definite with $\sup_{(x,y)} k((x,y),(x,y)) < \kappa$, we assert $\sum_{i=1}^n\beta_i^+ \rightarrow 1$ in probability as $n\rightarrow \infty$.
\end{corollary}

Next, we will relax the finite space condition on $\XX \times \YY$ in \thmref{thm:finiteconsist}. To this end, we introduce the following convenient concept of $\epsilon$-partition.
\begin{definition}[$\epsilon$-partition]
	An $\epsilon$-partition of a metric space $\XX$ is a partition whose elements are all within $\epsilon$-balls of $\XX$.
\end{definition}

Since a compact space is totally bounded, we have the more general result.
\begin{customthm}{{\ref{thm:consist2}}}
	Assume that $\XX$ is compact and $|\YY| < \infty$, $k$ is a strictly positive definite continuous kernel with $\sup_{(x,y)} k((x,y),(x,y)) < \kappa$ and $f(x,y) = \norm{\psi(y) - \mu(x)}_{\Hy}^2 \in \Hx\otimes\Hy$. With the conditions in \thmref{thm:consist1}, we assert that $\widehat{\mu}^+_{(X,Y)}$ is a consistent estimator of $\mu_{(X,Y)}$ and $\left|\Epp - \Ep_s[\mu]\right|\rightarrow 0$ in probability as $n\rightarrow \infty$.
\end{customthm}
\begin{proof}
	From the condition that $\phi(x,y)$ is continuous on the compact space $\XX\times\YY$, we know $\phi(x,y)$ and $\phi(x)$ are uniformly continuous. 
	
	For any probability measure $p$ and $\epsilon$-partition of $\XX$, we can construct a new discretized probability measure in the following way. Suppose the $\epsilon$-partition is $\{B_1^\epsilon,B_2^\epsilon,\cdots\}$, we identify each set $B_i^\epsilon$ with a representative element $x_i^c \in B_i^\epsilon$. The resulting probability measure is denoted as $p^\epsilon$ and satisfies $p^\epsilon(A) = \sum_{x_i^c\in A} p(B_i^\epsilon)$. We also define the discretization $x_i^\epsilon$ of $x_i$ to be $x_i^\epsilon = x_j^c$ if $x_i \in B_j$. Let the kernel embedding of $p$ be $\mu$ and $p^\epsilon$ be $\mu^\epsilon$. Suppose $\forall \delta > 0$, $\exists \epsilon > 0$ such that $\norm{x_1 - x_2} \leq \epsilon$ implies $\norm{\phi(x_1) - \phi(x_2)}_{\Hx} \leq \delta$.  We assert that $\norm{\mu - \mu^\epsilon} \leq  \delta$. To prove this, we observe that an \iid sample $\{x_1,\cdots, x_n\}$ from $p$ is also an \iid sample of $p^\epsilon$ if we replace $x_i$ with $x_i^\epsilon$. Since the estimator $\widehat{\mu} = \frac{1}{n} \sum_{i=1}^n \phi(x_i)$ is a consistent estimator of $\mu$, we know that $\widehat{\mu^\epsilon} = \frac{1}{n} \sum_{i=1}^n \phi(x_i^\epsilon)$ is also consistent. Via consistency, we have that with no less than any high probability $1-\Delta$, for any $n > N(\Delta,\delta',\epsilon)$, $\norm{\widehat{\mu} - \mu} \leq \delta'$ and $\norm{\widehat{\mu^\epsilon} - \mu^\epsilon} \leq \delta'$ holds. Since $\norm{\widehat{\mu} - \widehat{\mu^\epsilon}}\leq \frac{1}{n} \sum_{i=1}^n \norm{\phi(x_i) - \phi(x_i^\epsilon)}$ and $\norm{x_i - x_i^\epsilon} \leq \epsilon$, we have $\norm{\widehat{\mu} - \widehat{\mu^\epsilon}} \leq \delta$ from uniform continuity. Combining this with $\norm{\widehat{\mu} - \mu} \leq \delta'$ and $\norm{\widehat{\mu^\epsilon} - \mu^\epsilon} \leq \delta'$ we know $\norm{\mu - \mu^\epsilon} \leq \norm{\mu - \widehat{\mu}} + \norm{\widehat{\mu} - \widehat{\mu^\epsilon}} + \norm{\widehat{\mu^\epsilon} - \mu^\epsilon} \leq \delta + 2\delta'$ with high probability $1 - \Delta$. Note that $\norm{\mu - \mu^\epsilon} \leq \delta + 2\delta'$ is a deterministic event and holds for any $\delta'>0$, we have $\norm{\mu - \mu^\epsilon} \leq \delta$. 
	
	Now we would like to discretize $X$ for $\mu_{(X,Y)}$. For any $\epsilon > 0$, we have $\widehat{\mu}_{(X,Y)} = \sum_{i=1}^n \beta_i \phi(x_i,y_i) \rightarrow \mu_{(X,Y)}$ in probability and $\widehat{\mu^\epsilon}_{(X,Y)} = \sum_{i=1}^n \beta_i^\epsilon \phi(x_i^\epsilon,y_i) \rightarrow \mu_{(X,Y)}^\epsilon$ in probability. Since $\beta_i$ depends only on $y_1,\cdots,y_n$, we have $\beta_i = \beta_i^\epsilon$. From the last paragraph we suppose that $\epsilon$ is chosen such that $\forall \norm{(x_i^\epsilon,y_i) - (x_i,y_i)}\leq \epsilon$, $\norm{\phi(x_i^\epsilon,y_i) - \phi(x_i,y_i)}\leq \delta$. Note that 
	\begin{align*}
	\norm{\sum_{i=1}^n \beta_i^+ \phi(x_i,y_i) - \mu_{(X,Y)}} &\leq \norm{\sum_{i=1}^n \beta_i^+ \phi(x_i,y_i) - \sum_{i=1}^n \beta_i^+ \phi(x_i^\epsilon, y_i)}\\
	&+ \norm{\sum_{i=1}^n \beta_i^+ \phi(x_i^\epsilon, y_i) - \sum_{i=1}^n \beta_i \phi(x_i^\epsilon,y_i)}\\
	&+ \norm{\sum_{i=1}^n \beta_i \phi(x_i^\epsilon,y_i) - \mu_{(X,Y)}^\epsilon} + \norm{\mu_{(X,Y)}^\epsilon - \mu_{(X,Y)}}\\
	&\leq \delta\sum_{i=1}^n \beta_i^+ + \delta + \norm{\sum_{i=1}^n \beta_i^+ \phi(x_i^\epsilon, y_i) - \sum_{i=1}^n \beta_i \phi(x_i^\epsilon,y_i)}\\ &+ \norm{\sum_{i=1}^n \beta_i \phi(x_i^\epsilon,y_i) - \mu_{(X,Y)}^\epsilon}.
	\end{align*}
	From Corollary~\ref{col:betapsum}, \thmref{thm:finiteconsist} and the consistency of $\sum_{i=1}^n \beta_i \phi(x_i^\epsilon,y_i)$ we see $\norm{\sum_{i=1}^n \beta_i^+ \phi(x_i,y_i) - \mu_{(X,Y)}}$ can be arbitrarily small with arbitrarily high probability. This proves the consistency of $\sum_{i=1}^n \beta_i^+ \phi(x_i,y_i)$.
\end{proof}

\begin{corollary}\label{col:consistx}
Assume that $\XX$ is compact and $|\YY| < \infty$, $k$ is a bounded strictly positive definite continuous kernel, $k_\XX$ is a bounded kernel with $\sup_x k_\XX(x,x) \leq \kappa_\XX$, we assert that $\widehat{\mu}^+_X = \sum_{i=1}^n \beta_i^+ \phi(x_i)$ is a consistent estimator of $\mu_X$, \ie, the kernel embedding of the marginal distribution on $X$.
\end{corollary}

\begin{theorem}\label{thm:banach}
Let $\mcal{B}_1$, $\mcal{B}_2$ be Banach spaces. For any linear operator $\mcal{A}:\mcal{B}_1 \rightarrow \mcal{B}_2$, we assert that there exists a subset $\mcal{F} \subseteq \mcal{B}_1$ such that $\mcal{F}$ is dense in $\mcal{B}_1$ and $\norm{\mcal{A}f}_{\mcal{B}_2} \leq N\norm{f}_{\mcal{B}_1}$ for some constant $N$ and any $f\in\mcal{F}$.
\end{theorem}
\begin{proof}
Let $M_k$ be the set of $f\in\mcal{B}_1$ satisfying $\norm{\mcal{A}f}_{\mcal{B}_2} \leq k \norm{f}_{\mcal{B}_1}$. Clearly we have $\BB_1 = \bigcup_{k=1}^\infty M_k$. Since $\BB_1$ is complete, we can invoke Baire category theorem to conclude that there exists an integer $n$ such that $M_n$ is dense in some sphere $S_0\subseteq \BB_1$. Consider the spherical shell $P$ in $S_0$ consisting of the points $z$ for which
\begin{align*}
\beta < \norm{z - y_0} < \alpha,
\end{align*}
where $0<\beta<\alpha$, $y_0 \in M_n$. Next, translate the spherical shell $P$ so that its center coincides with the origin of coordinates to obtain spherical shell $P_0$. We now show that there is some set $M_N$ dense in $P_0$. For every $z \in M_n \cap P$, we have
\begin{multline*}
\norm{\mcal{A}(z - y_0)}_{\BB_2} \leq \norm{\mcal{A}z}_{\BB_2} + \norm{\mcal{A}y_0}_{\BB_2} \leq n(\norm{z}_{\BB_1} + \norm{y_0}_{\BB_1})\leq n(\norm{z - y_0}_{\BB_1} + 2\norm{y_0}_{\BB_1}) \\
= n\norm{z - y_0}_{\BB_1} [1 + 2\norm{y_0}_{\BB_1}/\norm{z - y_0}_{\BB_1}] \leq n\norm{z - y_0}_{\BB_1} [1 + 2\norm{y_0}_{\BB_1}/\beta].
\end{multline*}
Let $N = n(1+ 2\norm{y_0}_{\BB_1}/\beta)$, we have $z -y_0 \in M_N$. Since $z-y_0 \in M_N$ is obtained from $z\in M_n$ and $M_n$ is dense in $P$, it is easy to see that $M_N$ is dense in $P_0$. For any $y \in \BB_1$ except $\norm{y}_{\BB_1} = 0$, it is always possible to choose $\lambda$ so that $\beta < \norm{\lambda y} < \alpha$ and we can construct a sequence $y_k \in M_N$ that converges to $\lambda y$. This means there exists a sequence $(1/\lambda) y_k$ converging to $y$. By virtue of $(1/\lambda )y_k \in M_N$ and $0\in M_N$, we conclude $M_N$ is dense in $\BB_1$.
\end{proof}
\begin{theorem}\label{thm:conc}
Let $(\Omega,\mcal{F},P)$ be a probability space and $\xi$ be a random variable on $\Omega$ taking values in a Hilbert space $\mcal{K}$. Define $\mcal{A}:f\in\mcal{K} \mapsto \langle f,\xi(\cdot)\rangle \in \HH$, where $\HH$ is a RKHS with feature maps $\phi(\omega)$. Let $\mu$ be a kernel embedding for $P^\pi$ and $\widehat{\mu} = \sum_{i=1}^n\beta_i^+ \phi(\omega_i)$ be a consistent estimator of $\mu$. Assume $\sum_{i=1}^n \beta_i^+ \rightarrow 1$ in probability and there are two positive constants $H$ and $\sigma$ such that $\norm{\xi(\omega)}_\mcal{K} \leq \frac{H}{2}$ a.s. and $\mbb{E}_{P^\pi}[\norm{\xi}_{\mcal{K}}^2]\leq \sigma^2$. Then for any $\epsilon > 0$,
\begin{align}
\lim_{n\rightarrow 0}P^l \left[(\omega_1,\cdots,\omega_l)\in\Omega^l \mid  \norm{\sum_{i=1}^n \beta_i^+ \xi(w_i) - \mbb{E}_{P^\pi}[\xi]}_{\mcal{K}} > \epsilon\right] = 0
\end{align}
\end{theorem}
\begin{proof}
From the consistency of $\widehat{\mu}$, we know for every $\epsilon_1$, there exists $N_{\epsilon_1}(\delta_1)$ such that $\forall n>N_{\epsilon_1}(\delta_1)$, $\norm{\widehat{\mu} - \mu}_{\HH} < \epsilon_1$ with probability no less than $1-\delta_1$. Similarly, for every $\epsilon_2$, there exists $N_{\epsilon_2}(\delta_2)$ such that $\forall n>N_{\epsilon_2}(\delta_2)$, $\left|\sum_{i=1}^n \beta_i^+ - 1\right| < \epsilon_2$ with probability no less than $1-\delta_2$. Furthermore, with probability no less than $1-\delta_2$, $\norm{\sum_{i=1}^n \beta_i^+ \xi(w_i) - \mbb{E}_{P^\pi}[\xi]}_{\mcal{K}} \leq \sum_{i=1}^n \beta_i^+ \norm{\xi(w_i)}_{\mcal{K}} +\norm{\mbb{E}_{P^\pi}[\xi]}_{\mcal{K}} \leq (1+\epsilon_2)\norm{\xi(\omega)}_{\mcal{K}}+ \mbb{E}_{P^\pi}[\norm{\xi}] \leq \frac{H(1+\epsilon_2)}{2} + \sqrt{\mbb{E}_{P^\pi}[\norm{\xi}_{\mcal{K}}^2]} = \frac{H(1+\epsilon_2)}{2} + \sigma$, where the last two inequalities follow from Jensen's inequality.

Let $f = \sum_{i=1}^n \beta_i^+ \xi(w_i) - \mbb{E}_{P^\pi}[\xi]$ and clearly $\norm{f}_{\mcal{K}}\leq \frac{H(1+\epsilon_2)}{2} + \sigma$. Consider $\Delta_f := \sum_{i=1}^n \beta_i^+ \langle f,\xi(\omega_i)\rangle - \langle f,\mbb{E}_{P^\pi}[\xi]\rangle = \sum_{i=1}^n \beta_i^+ [\mcal{A}f](\omega_i) - \mbb{E}_{P^\pi}[\mcal{A}f] = \langle \widehat{\mu} - \mu, \mcal{A}f\rangle $. In virtue of \thmref{thm:banach}, for any $\epsilon_3$, there exists an element $g \in \mcal{K}$ and constant $N$ (only depends on $\mcal{A}$) such that $\norm{g -f}_{\mcal{K}} < \epsilon_3$ and $\norm{\mcal{A} g}_{\HH} \leq N \norm{g}_{\mcal{K}}$. Similarly define $\Delta_g := \sum_{i=1}^n \beta_i^+ \langle g,\xi(\omega_i)\rangle - \langle g,\mbb{E}_{P^\pi}[\xi]\rangle = \langle \widehat{\mu} - \mu, \mcal{A}g\rangle$. It is easy to see that $|\Delta_g - \Delta_f| \leq (1+\epsilon_2)\epsilon_3\norm{\xi(\omega)}_{\mcal{K}} + \epsilon_3 \norm{\mbb{E}_{P^\pi}[\xi]}_{\mcal{K}} \leq \frac{H\epsilon_3(1+\epsilon_2)}{2} + \epsilon_3 \sigma$ and $\Delta_g = \langle \widehat{\mu} - \mu, \mcal{A}g\rangle \leq \epsilon_1 N \norm{g}_{\mcal{K}} \leq \epsilon_1 N (\epsilon_3 + \norm{f}_{\mcal{K}}) \leq \epsilon_1 N (\epsilon_3 + \sigma + \frac{H(1+\epsilon_2)}{2})$ with probability no less than $1-\delta_1-\delta_2$. Hence $\norm{\sum_{i=1}^n \beta_i^+ \xi(w_i) - \mbb{E}_{P^\pi}[\xi]}_{\mcal{K}}=\sqrt{|\Delta_f|} \leq \sqrt{\epsilon_1 N (\epsilon_3 + \sigma + \frac{H(1+\epsilon_2)}{2}) + \frac{H\epsilon_3(1+\epsilon_2)}{2} + \epsilon_3 \sigma}$ with probability no less than $1-\delta_1-\delta_2$ for all $n > \max(N_{\epsilon_1}(\delta_1),N_{\epsilon_2}(\delta_2))$. The theorem now gets proved.
\end{proof}

The proof of \thmref{thm:diff} is based on the proof of Thm.~5 in \cite{de2005risk}, with more assumptions and different concentration results. For convenience, we borrow some notations in their paper and refer the readers to \cite{de2005risk} for definitions. We suggest the readers to be familiar with \cite{de2005risk} because we modify and skip some details of the proofs to make the reasoning clearer.

Let $\mcal{X},\YY$ be Polish spaces, $\Hy$ be a separable Hilbert space, $\ZZ = \XX\times \YY$, $\HK$ be a real Hilbert space of functions $\mu:\XX\rightarrow \Hy$ satisfying $\mu(x) = K_x^*\mu$ where $K_x:\Hy\rightarrow\HK$ is the bounded operator $K_x v=K(\cdot,x)v,\quad v\in\Hy$. Moreover, let $T_x = K_xK_x^* \in \mcal{L}_2(\HK)$ be a positive Hilbert-Schmidt operator. 

Let $\rho$ be a probability measure on $\ZZ$ and $\rho_X$ denotes the marginal distribution of $\rho$ on $\XX$. We suppose that $\rho = p(X\mid Y)\pi(Y)$ and thus it incorporates the information of the prior. In contrast, we are given a sample $\zz = ((x_1,y_1),\cdots,(x_n,y_n))$ from another distribution on $\ZZ$ with the same $p(X\mid Y)$.

The optimization objective now becomes $\Ep_s[\mu] = \int_{\ZZ} \norm{\mu(x) - \phi(y)}_{\Hy}^2 d\rho(x,y)$. Denote $T = \int_\XX T_x d \rho_X(x)$, $T_\xx = \sum_{i=1}^n \beta_i^+ T_{x_i}$, $\mu_{\HK} = \argmin_{\mu} \Ep_s[\mu]$, $\mu^\lambda = \Ep_s[\mu] + \lambda \norm{\mu}_{\HK}^2$ and $\mu_{\zz}^\lambda = \Eln$. Additionally, let $A:\HK\rightarrow L^2(\ZZ,\rho,\Hy)$ be the linear operator $(Af)(x,y) = K_x^* f\quad \forall(x,y)\in\ZZ$ and $A_{\zz} := A_{\rho = \sum_{i=1}^n\beta_i^+ \delta_{x_i}}$. Finally, let $\mcal{A}(\lambda) = \norm{\mu^\lambda - \mu_{\HK}}_\rho^2 = \norm{\sqrt{T}(\mu^\lambda - \mu_{\HK})}$, $\mcal{B}(\lambda) = \norm{\mu^\lambda - \mu_{\HK}}_{\HK}^2$ and $\mcal{N}(\lambda) = \up{Tr}((T+\lambda)^{-1}T)$.

\begin{assumption}
Let $\mcal{A}_1:f\in \mcal{L}_2(\HK) \mapsto \langle f,(T+\lambda)^{-1}T_{\cdot} \rangle \in \HH_1$, $\mcal{A}_2: f\in\mcal{L}(\HK) \mapsto \langle f,T_{\cdot}(\mu^\lambda - \mu_{\HK}) \in \HH_2$, $\mcal{A}_3: f\in \HK \mapsto \langle f,(T+\lambda)^{-\frac{1}{2}}K_{\text{\# 1}}(\psi(\text{\# 2}) - \mu_{\HK}(\text{\# 1}))\rangle \in \HH_3$, where $\# 1$ and $\# 2$ denote two arguments of the function. We assume that $\HH_1=\HH_2=\HH_{\XX}$, $\HH_3=\HH_{\XX}\otimes\HH_{\YY}$.
\end{assumption}
\begin{assumption}
 We assume that $\widehat{\mu}_{(X,Y)}^+ = \sum_{i=1}^n \beta_i^+ \phi(x_i)\otimes\psi(y_i) $ is a consistent estimator of $\mu_{(X,Y)}$ and $\widehat{\mu}_{\XX}^+ = \sum_{i=1}^n \beta_i^+ \phi(x_i)$ is also consistent for the kernel embedding of the marginal distribution on $X$. Furthermore, we assume $\sum_{i=1}^n \beta_i^+ \xrightarrow{\text{p}} 1$. Note that as shown in \thmref{thm:consist2}, \thmref{thm:betas} and  Corollary~\ref{col:consistx}, this hypothesis holds when $\XX$ is compact and $\YY$ is finite.
\end{assumption}

\begin{theorem}\label{thm:app1}
With the above Assumption~1, Assumption~2 and Hypothesis~1, Hypothesis~2 in \cite{de2005risk}, we assert that if $\lambda_n$ decreases to 0, 
\begin{align}
\Ep_s[\mu_{\zz}^{\lambda_n}] - \Ep_s[\mu_{\HK}] \rightarrow 0
\end{align}
in probability as $n\rightarrow \infty$.
\end{theorem}
\begin{proof}
This proof is adapted from that of Thm.~5 in \cite{de2005risk}. We split the proof to 3 steps. 

\textbf{Step 1}: Given a training set $\zz = (\xx,\yy)\in\ZZ^n$, Prop.~2 in \cite{de2005risk} gives
\begin{align*}
\Ep_s[\mu_{\zz}^\lambda] - \Ep_s[\mu_{\HK}] = \norm{\sqrt{T}(\mu_{\zz}^\lambda - \mu_{\HK})}_{\HK}^2.
\end{align*}
As usual,
\begin{align*}
\mu_{\zz}^\lambda - \mu_{\HK} = (\mu_{\zz}^\lambda - \mu^\lambda) + (\mu^\lambda- \mu_{\HK})
\end{align*}
Another application of Prop.~2 in \cite{de2005risk} gives
\begin{align*}
\mu_{\zz}^\lambda - \mu^\lambda &= (T_{\xx} + \lambda)^{-1} A_{\zz}^* \psi(\yy) - (T+\lambda)^{-1}A^* \psi(y)\\
&= (T_{\xx} + \lambda)^{-1}(A_{\zz}^* \psi(\yy) - T_{\xx}\mu_{\HK}) + (T_{\xx} + \lambda)^{-1}(T-T_{\xx})(\mu^\lambda - \mu_{\HK}).
\end{align*}
From $\norm{\mu_1+\mu_2+\mu_3}_{\HK}^2 \leq 3(\norm{\mu_1}_{\HK}^2+\norm{\mu_2}_{\HK}^2+\norm{\mu_3}_{\HK}^2)$,
\begin{align}
\Ep_s[\mu_{\zz}^\lambda] - \Ep_s[\mu_{\HK}] \leq 3(\mcal{A}(\lambda) + \mcal{S}_1(\lambda,\zz)+\mcal{S}_2(\lambda,\zz))\label{eqn:b1},
\end{align}
where
\begin{align*}
\mcal{S}_1(\lambda,\zz) &= \norm{\sqrt{T}(T_{\xx}+\lambda)^{-1}(A_{\zz}^* \psi(\yy) - T_{\xx}\mu_{\HK})}_{\HK}^2\\
\mcal{S}_2(\lambda,\zz) &= \norm{\sqrt{T}(T_{\xx} + \lambda)^{-1}(T-T_{\xx})(\mu^\lambda - \mu_{\HK})}_{\HK}^2.
\end{align*}

\textbf{Step 2}: probabilistic bound on $\mcal{S}_2(\lambda,\zz)$. First
\begin{align}
\mcal{S}_2(\lambda,\zz) \leq \norm{\sqrt{T}(T_{\xx} + \lambda)^{-1}}_{\mcal{L}(\HK)}^2 \norm{(T-T_{\xx})(\mu^\lambda - \mu_{\HK})}_{\HK}^2.
\end{align}

\textbf{Step 2.1}: probabilistic bound on $\norm{\sqrt{T}(T_{\xx}+\lambda)^{-1}}_{\mcal{L}(\HK)}$. We introduce an auxiliary quantity
\begin{align*}
\Theta(\lambda,\zz) = \norm{(T+\lambda)^{-1}(T-T_{\xx})}_{\mcal{L}(\HK)}
\end{align*}
and assume
\begin{align*}
\Theta(\lambda,\zz) \leq \frac{1}{2}.
\end{align*}
Invoking the Neumann series,
\begin{align}
\norm{\sqrt{T}(T_{\xx} + \lambda)^{-1}}_{\mcal{L}(\HK)} &= \sqrt{T}(T+\lambda)^{-1}\sum_{n=0}^{\infty}((T+\lambda)^{-1}(T-T_{\xx}))^n\notag \\
&\leq \norm{\sqrt{T}(T+\lambda)^{-1}}_{\mcal{L}_{\HK}} \sum_{n=0}^\infty \Theta(\lambda,n)^n\notag \\
(\text{By spectral theorem})\quad &\leq \frac{1}{2\sqrt{\lambda}}\frac{1}{1-\Theta(\lambda,\zz)} \leq \frac{1}{\sqrt{\lambda}}\label{eqn:b2}
\end{align}

We now claim that $\Theta(\lambda,\zz) \leq \frac{1}{2}$ with high probability as $n\rightarrow \infty$. Let $\xi_1:\XX\rightarrow \mcal{L}_2(\HK)$ be the random variable
\begin{align*}
\xi_1(x) = (T+\lambda)^{-1}T_{x}.
\end{align*}
By the same reasoning in the proof of Thm.~5 in \cite{de2005risk}, we have $\norm{\xi_1}_{\mcal{L}_2(\HK)} \leq \frac{\kappa}{\lambda} = \frac{H_1}{2}$ and $\mbb{E}[\norm{\xi_1}_{\mcal{L}_2(\HK)}^2] \leq \frac{\kappa}{\lambda}\mcal{N}(\lambda) = \sigma_1^2$. Our assumptions and \thmref{thm:conc} ensure that for any $\delta_1$ there exists $N_1(\delta_1)$ such that
\begin{align*}
\Theta(\lambda,\zz) = \norm{(T+\lambda)^{-1}T_{\xx} - (T+\lambda)^{-1}T}_{\mcal{L}_2(\HK)} \leq \frac{1}{2}
\end{align*}
with probability greater than $1-\delta_1$ as long as $n > N_1(\delta_1)$.

\textbf{Step 2.2}: probabilistic bound on $\norm{(T-T_{\xx})(\mu^\lambda - \mu_{\HK})}_{\mcal{L}(\HK)}$. Let $\xi_2:\XX\rightarrow\HK$ be the random variable 
\begin{align*}
\xi_2(x) = T_x(\mu^\lambda - \mu_{\HK}).
\end{align*}
By the same reasoning, we have $\norm{\xi_2(x)}_{\HK}\leq \kappa \sqrt{\mcal{B}(\lambda)}=\frac{H_2}{2}$ and $\mbb{E}[\norm{\xi_2}_{\HK}^2] \leq \kappa \mcal{A}(\lambda) = \sigma_2^2$. Applying our assumptions and \thmref{thm:conc} we conclude that for any $\delta_2, \epsilon_2$ there exists $N_2(\delta_2,\epsilon_2)$ such that
\begin{align}
\norm{(T-T_{\xx})(\mu^\lambda - \mu_{\HK})}_{\HK} \leq \epsilon_2\label{eqn:b3}
\end{align}
with probability greater than $1-\delta_2$ as long as $n > N_2(\delta_2,\epsilon_2)$.

\textbf{Step 3}: probabilistic bound on $\mcal{S}_1(\lambda,\zz)$. As usual, 
\begin{align*}
\mcal{S}_1(\lambda,\zz) \leq \norm{\sqrt{T}(T_{\xx} + \lambda)^{-1}(T+\lambda)^{1/2}}_{\mcal{L}(\HK)}^2\norm{(T+\lambda)^{-1/2}(A_{\zz}^* \psi(\yy) - T_{\xx}\mu_{\HK})}_{\HK}^2.
\end{align*}

\textbf{Step 3.1}: bound $\norm{\sqrt{T}(T_{\xx}+\lambda)^{-1}(T+\lambda)^{1/2}}_{\mcal{L}(\HK)}$. Let 
\begin{align*}
\Omega(\lambda,\zz) = \norm{(T+\lambda)^{1/2} (T-T_{\xx})(T+\lambda)^{-1/2}}_{\mcal{L}(\HK)}
\end{align*}
and assume $\Omega(\lambda,\zz) \leq \frac{1}{2}$. Clearly, 
\begin{align}
&~~~~~\norm{\sqrt{T}(T_{\xx}+\lambda)^{-1}(T+\lambda)^{1/2}}_{\mcal{L}(\HK)}\\
&= \norm{\sqrt{T}(T+\lambda)^{-1/2}\{I - (T+\lambda)^{1/2}(T-T_{\xx})(T+\lambda)^{-1/2}\}^{-1}}_{\mcal{L}(\HK)}\notag \\
&\leq \norm{\sqrt{T}(T+\lambda)^{-1/2}}_{\mcal{L}(\HK)}\sum_{i=1}^\infty \Omega(\lambda,\zz)^n\notag\\
(\text{By spectral theorem}) \quad &\leq \frac{1}{1-\Omega(\lambda,\zz)} = 2.\label{eqn:b4}
\end{align}
On the other hand,
\begin{align*}
\Omega(\lambda,\zz)^2 &= \langle (T+\lambda)^{-1}(T-T_{\xx}),((T+\lambda)^{-1}(T-T_{\xx}))^*\rangle_{\mcal{L}_2(\HK)}\\
&\leq \norm{(T+\lambda)^{-1}(T-T_{\xx})}_{\mcal{L}_2(\HK)}^2 = \Theta(\lambda,\zz)^2.
\end{align*}
As a result, we have $\Omega(\lambda,\zz)\leq \frac{1}{2}$ with probability greater than $1-\delta_1$ as long as $n > N_1(\delta_1)$.

\textbf{Step 3.2}: probabilistic bound on $\norm{(T+\lambda)^{-1/2}(A_{\zz}^*\psi(\yy) - T_{\xx}\mu_{\HK})}_{\HK}$. Let $\xi_3: \ZZ \rightarrow \HK$ be the random variable
\begin{align*}
\xi_3(x,y) = (T+\lambda)^{-1/2} K_x (\psi(y)-\mu_{\HK}(x)).
\end{align*}
Via the same reasoning in the proof of Thm.~5 in \cite{de2005risk}, we have $\norm{\xi_3}_{\HK} \leq \sqrt{\frac{\kappa M}{\lambda}}=\frac{H_3}{2}$ and $\mbb{E}[\norm{\xi_3}_{\HK}^2] \leq M\mcal{N}(\lambda) = \sigma_3^2$. From our assumptions and \thmref{thm:conc} we know for each $\epsilon_3$ and $\delta_3$ there exists $N_3(\delta_3,\epsilon_3)$ such that
\begin{align}
\norm{(T+\lambda)^{-1/2}(A_{\zz}^*\psi(\yy) - T_{\xx}\mu_{\HK})}_{\HK} \leq \epsilon_3\label{eqn:b5}
\end{align}
with probability greater than $1-\delta_3$ as long as $n > N_3(\delta_3,\epsilon_3)$.

Linking bounds \eqref{eqn:b1}, \eqref{eqn:b2}, \eqref{eqn:b3}, \eqref{eqn:b4}, and \eqref{eqn:b5} we obtain that for every $\epsilon_1,\epsilon_2,\epsilon_3 > 0$ and  $\delta_1,\delta_2,\delta_3 > 0$ there exists $N = \max\{N_1(\delta_1), N_2(\delta_2,\epsilon_2),N_3(\delta_3,\epsilon_3)\}$ such that for each $n > N$,
\begin{align*}
\Ep_s[\mu_{\zz}^\lambda] - \Ep_s[\mu_{\HK}] \leq 3[\mcal{A}(\lambda) + \frac{\epsilon_2^2}{\lambda}+4\epsilon_3^2 ]
\end{align*}
with probability greater than $1-\delta_1-\delta_2-\delta_3$.
This means that for any $\epsilon > 0$ and fixed $\lambda$
\begin{align}
\lim_{n\rightarrow 0} p\left(\Ep_s[\mu_{\zz}^\lambda] - \Ep_s[\mu_{\HK}] > 3\mcal{A}(\lambda) + \epsilon \right) = 0\label{eqn:l1}
\end{align}
From \cite{engl1996regularization} we know 
\begin{align}
\lim_{\lambda \rightarrow 0} \mcal{A}(\lambda) = 0.\label{eqn:l2}
\end{align}
Combining \eqref{eqn:l1} and \eqref{eqn:l2} we can conclude that as long as $\lambda$ decreases to $0$, $\Ep_s[\mu_{\zz}^\lambda]$ converges to $\Ep_s[\mu_{\HK}]$ in probability.
\end{proof}
\begin{customthm}{\ref{thm:diff}}
Assume Hypothesis~1 and Hypothesis~2 in \cite{de2005risk} and our Assumption~1 hold. With the conditions in \thmref{thm:consist2}, we assert that if $\lambda_n$ decreases to 0 sufficiently slowly, 
\begin{align}
\Ep_s[\widehat{\mu}_{\lambda_n,n}] - \Ep_s[\mu'] \rightarrow 0
\end{align}
in probability as $n\rightarrow \infty$.
\end{customthm}
\begin{proof}
The proof follows directly from \thmref{thm:consist1}, \thmref{thm:consist2}, \thmref{thm:betas}, Corollary~\ref{col:consistx} and \thmref{thm:app1}.
\end{proof}
\end{document}